\documentclass[numbers,webpdf]{ima-authoring-template}%

\usepackage{graphicx} 

\graphicspath{{Fig/}}

\theoremstyle{thmstyletwo}%
\newtheorem{theorem}{Theorem}[section]
\newtheorem{proposition}[theorem]{Proposition}%

\newtheorem{example}[theorem]{Example}%
\newtheorem{remark}[theorem]{Remark}%
\newtheorem{definition}[theorem]{Definition}
\newtheorem{lemma}[theorem]{Lemma}
\newtheorem{assumption}[theorem]{Assumption}
\newtheorem{corollary}[theorem]{Corollary}

\usepackage{amsthm,amssymb}
\usepackage{mathrsfs}
\usepackage{mathtools,thmtools,thm-restate}
\usepackage{float}
\usepackage{graphicx}

\usepackage{hyperref}
\usepackage{esint}
\usepackage{fixltx2e}

\usepackage[pdftex,dvipsnames]{xcolor}  
\numberwithin{equation}{section}

\newcommand{\V}{\mathsf{\mathbf{V}}}
\newcommand{\A}{\mathsf{\mathbf{A}}}
\newcommand{\U}{\mathsf{\mathbf{U}}}
\newcommand{\NN}{\mathsf{NN}}

\newcommand{\R}{\mathbb{R}}
\newcommand{\C}{\mathbb{C}}
\newcommand{\E}{\mathbb{E}}
\newcommand{\N}{\mathbb{N}}

\renewcommand{\P}{\mathbb{P}}

\newcommand{\cA}{\mathcal{A}}

\newcommand{\cE}{\mathcal{E}}
\newcommand{\cF}{\mathcal{F}}
\newcommand{\cG}{\mathcal{G}}

\newcommand{\cK}{\mathcal{K}}
\newcommand{\cL}{\mathcal{L}}

\newcommand{\cP}{\mathcal{P}}
\newcommand{\cQ}{\mathcal{Q}}
\newcommand{\cR}{\mathcal{R}}

\newcommand{\cV}{\mathcal{V}}

\newcommand{\cX}{\mathcal{X}}
\newcommand{\cY}{\mathcal{Y}}

\newcommand{\tR}{\tilde{\cR}}

\newcommand{\dz}{\;\mathsf{d}z}
\newcommand{\dx}{\;\mathsf{d}x}
\newcommand{\dy}{\;\mathsf{d}y}
\newcommand{\dt}{\; \mathsf{d}t}

\newcommand{\dc}{{d_c}}

\newcommand{\FNO}{\mathrm{FNO}}
\newcommand{\NO}{\mathrm{NO}}
\newcommand{\DON}{\mathrm{DON}}

\newcommand{\Alg}{\mathrm{Alg}}

\newcommand{\Prob}{\mathrm{Prob}}

\newcommand{\cond}{\,|\,}
\newcommand{\id}{\mathrm{id}}

\newcommand{\define}{\textbf}

\newcommand{\unif}{\mathrm{Unif}}
\newcommand{\Unif}{\unif}

\renewcommand{\tilde}{\widetilde}
\renewcommand{\hat}{\widehat}

\DeclareMathOperator*{\essinf}{ess\,inf}
\DeclareMathOperator*{\esssup}{ess\,sup}

\newcommand{\aint}{\fint}
\newcommand\p{\partial}
\newcommand{\set}[2]{{\left\{ #1 \,\middle|\, #2 \right\}}}
\newcommand{\slot}{{\,\cdot\,}}
\newcommand{\supp}{\mathrm{supp}}

\newcommand{\Lip}{\mathrm{Lip}}

\newcommand{\Span}{\mathrm{span}}

\renewcommand{\L}{{\lfloor \bl^\ast/2 \rfloor}}
\newcommand{\bl}{\mathbf{\ell}}

\renewcommand{\det}{\mathrm{det}}
\newcommand{\mt}[1]{\textcolor{black}{#1}}

\begin{document}

\copyrightyear{2026}
\firstpage{1}


\title[Theory-to-Practice Gap for NN and NO]{Theory-to-Practice Gap for Neural Networks and Neural Operators}

\author{Philipp Grohs\ORCID{0000-0001-9205-0969}
\address{\orgdiv{Department of Mathematics}, \orgname{University of Vienna}, \orgaddress{1090 Wien}, \state{Vienna}, \country{Austria}}}
\author{Samuel Lanthaler\ORCID{0000-0003-1911-246X}
\address{\orgname{G-Research}, \orgaddress{\street{1 Soho Pl}, \state{London}, \country{United Kingdom}}}}
\author{Margaret Trautner*\ORCID{0000-0001-9937-8393}
\address{\orgdiv{Department of Computing and Mathematical Sciences}, \orgname{California Institute of Technology}, \orgaddress{\street{1200 E California Blvd}, \state{Pasadena, CA}, \country{USA}}}}

\authormark{Grohs, Lanthaler, Trautner}

\corresp[*]{Corresponding author: \href{mailto:trautner@utexas.edu}{trautner@utexas.edu}}

\received{14}{11}{2025}


\abstract{This work studies the sample complexity of learning with ReLU neural networks and neural operators. For mappings belonging to relevant approximation spaces, we derive upper bounds on the best-possible convergence rate of any learning algorithm, with respect to the number of samples. In the finite-dimensional case, these bounds imply a gap between the parametric and sample complexities of learning, known as the \emph{theory-to-practice gap}. In this work, a unified treatment of the theory-to-practice gap is achieved in a general $L^p$-setting, while at the same time improving available bounds in the literature. Furthermore, based on these results the theory-to-practice gap is extended to the infinite-dimensional setting of operator learning. Our results apply to Deep Operator Networks and integral kernel-based neural operators, including the Fourier neural operator. We show that the best-possible convergence rate in a Bochner $L^p$-norm is bounded by rates of order $1/p$. }


\maketitle

\section{Introduction}

Deep learning has had remarkable success in a wide range of tasks such as speech recognition, computer vision, or natural language processing \cite{goodfellow2016deep}. Increasingly this methodology is also used in the scientific domain, with applications to protein folding \cite{jumper2021highly}, plasma physics \cite{degrave2022magnetic} and numerical weather prediction \cite{kurth2023fourcastnet,price2024probabilistic,bodnar2025foundation}. However, the theoretical underpinnings of this field remain incomplete, and significant advances are still required to understand the empirical success of neural networks in these diverse applications. In many of these tasks, the goal is to approximate an unknown mapping based on a dataset consisting of input- and output-pairs, so-called \emph{supervised learning}. This has motivated a surge of theoretical work aimed at deepening our understanding of supervised learning with neural networks.

Theoretical error estimates for supervised learning are usually obtained by splitting the overall error into two contributions \cite{cucker2002mathematical}: the approximation error and the generalization error. The approximation error measures the best-possible error that can be achieved by a given architecture. The generalization error bounds the difference between this best-possible error and the error achieved by empirical risk minimization, i.e. optimization based on a finite number of empirical data samples. This decomposition captures a trade-off between \emph{model expressivity} (or parametric complexity), and generalization from a finite amount of data, i.e. the \emph{sample complexity}.

The study of the expressivity of neural networks dates back several decades to foundational work such as that of Cybenko \cite{cybenko1989approximation} and Hornik, Stinchcombe, and White \cite{hornik1989multilayer}, which focused on qualitative universality theorems. Motivated by the need to better explain the empirical efficiency of deep neural networks in applications, there has recently been increased interest in deriving quantitative approximation guarantees. These relate achievable approximation errors to key factors, such as the depth, width or choice of activation function \cite{bolcskei2019optimal,YAROTSKY2017103,devore2021neural,elbrachter2021deep,grohs2023proof}. What is striking about these results is that the derived approximation rates are generally far superior to the rates that are achievable by classical numerical representations. This fact makes deep learning methods potentially appealing for applications in numerical analysis where a high convergence rate is often desired. 

\paragraph{The Theory-to-Practice Gap. }
Despite these encouraging results, any actual numerical approximation algorithm must operate with limited information on the function to be approximated -- typically in the form of a finite number of samples of the function itself, or samples of a local (differential) operator applied to the function. It is therefore a key question whether the theoretically established approximation rates can be retained under such limited information. 

For this reason, the sample complexity of deep learning has also been of recent focus. Of particular relevance to the present work are the articles \cite{grohs2023proof,berner2023learning,abdeljawad2023}. An informal summary of their results is as follows. First, define by $U^\alpha([0,1]^d)$ the set of functions on $[0,1]^d$ which, for any $n\in \N$, can be approximated by a neural network $\psi_n$ with at most $n$ non-zero weights, and with approximation error $\Vert f - \psi_n \Vert_{L^\infty} \le  n^{-\alpha}$. We refer to the approximation rate $\alpha$ as the \emph{parametric convergence rate}, since this rate holds with respect to the number of neural network parameters $n$.

Second, to address the sample complexity of computing such an approximation, \cite{grohs2023proof} considers (optimal) reconstruction methods aiming to reconstruct $f\in U^{\alpha}([0,1]^d)$ from point samples $f(x_1),\dots, f(x_N)$, and examines limits on the best possible guaranteed convergence rate -- \emph{not in terms of the number of parameters but in the number of samples $N$}. 

More formally, \cite{grohs2023proof} asks for the optimal convergence rate $\beta_\ast>0$ for which there exists a reconstruction method $A: f\mapsto Q(f(x_1),\dots , f(x_N))$ defined in terms of a reconstruction mapping $Q:\mathbb{R}^N\to L^p([0,1]^d)$, with a convergence guarantee of the form $\sup_{f\in U^\alpha}\Vert f - A(f) \Vert_{L^p} \le C N^{-\beta_\ast}$. Compared to the parametric convergence rate $\alpha$, which describes the optimal convergence rate in terms of the number of parameters, the rate $\beta_\ast$ is now in terms of the available information -- the $N$ function samples. This is of practical relevance, since a single function evaluation requires a certain minimal computational time, any upper bound on $\beta_\ast$ yields a corresponding lower bound on the time to compute an accurate approximation for $f\in U^\alpha$.

A naive counting argument based on the number of degrees of freedom would suggest that $\beta_\ast = \alpha$ coincides with the optimal parametric convergence rate $\alpha$. Indeed, it might be hoped that the determination of $n$ parameters of the approximating neural network $\psi_n$ requires $N=n$ point evaluations, implying that for $A(f) := \psi_n$, we have $\Vert f- A(f) \Vert_{L^\infty} \le n^{-\alpha} = N^{-\alpha}$. As is well-known, this intuition is indeed correct for reconstruction by several popular methods, including polynomials, trigonometric polynomials and certain kernel methods \cite{UlrichSamplingKernel}. However, the results of \cite{grohs2023proof} show that this intuition does not carry over to the sample complexity of neural network approximation spaces: In this case, there is a gap between $\beta_\ast$ and $\alpha$. In fact, even in the limit $\alpha \to \infty$, the optimal sample convergence rate $\beta_\ast$ remains uniformly bounded and, for $p=\infty$ it even holds that $\beta_\ast \lesssim \frac1d$ which implies the existence of a curse of dimension. The gap between $\alpha$ and $\beta_\ast$ is termed the \emph{theory-to-practice gap} as it describes the discrepancy between the complexity of a theoretically possible approximation and one that can actually be computed from the information at hand. 

A first contribution of this work is Theorem \ref{thm:main} which further extends and sharpens the bounds from \cite{grohs2023proof} in the setting of finite-dimensional function approximation. Roughly speaking, we show that in a general $L^p$-setting, $\beta_\ast \le \frac{1}{p} + \frac{1}{d}$ which means that for high input dimensions (i.e., large $d$), no actual algorithm is capable of beating the standard Monte-Carlo rate $\frac1p$ -- irrespective of the parametric convergence rate $\alpha$. Compared to results in \cite{grohs2023proof} which, roughly speaking, established bounds of the form $\beta_\ast \le \frac{1}{p} + 1$, this constitutes a significant improvement if the input dimension $d$ is large.

\paragraph{Operator Learning. }
In addition to the aforementioned works on neural networks in finite-dimensions, there has also been increasing interest in operator learning \cite{kls2024hna}. The aim of operator learning is to approximate non-linear operators $\cG: \cX \to \cY$, mapping between infinite-dimensional function spaces $\cX$ and $\cY$. In applications, such operators often arise as solution operators associated with a partial differential equation, but more general classes of operators can be considered. To approximate such $\cG$, operator learning frameworks generalize neural networks to this infinite-dimensional setting. Empirically, the most successful approach is usually based on supervised learning from training data, and hence the same questions as discussed above also arise in this context. 

Early work on operator learning dates back to a foundational paper by Chen and Chen \cite{chen1995universal}. Without any claim of completeness, we mention a number of works studying the parametric complexity of operator learning which aim to relate the model size to the achieved approximation accuracy, including both upper bounds on the required number of parameters, e.g. \cite{herrmann2024neural,KLM_JMLR2020,franco2023approximation,marcati2023}, as well as lower bounds, e.g. \cite{mhaskar1997neural,lanthaler2023pcanet,lanthaler2024parametric,lanthaler2024lipschitzoperators}. 
The data (or sample) complexity of operator learning has e.g. been studied in \cite{adcock2024optimal,liu2024deep,adcock2024learninglipschitzoperatorsrespect}. \mt{While the present work is restricted to the noiseless setting, some related work has investigated the sample complexity of operator learning in the setting of noisy data \cite{chen2026, de2023convergence, JMLR:v25:22-0719,reinhardt2024statistical, subedi2025operator}.} In connection with the present work, we also highlight \cite{NLM2024data}, where approximation spaces for the Fourier neural operator are introduced, and upper bounds on the sample complexity of learning on these spaces are discussed. Closely related spaces will be discussed in this work. These spaces, which are an infinite-dimensional generalization of the approximation spaces $U^\alpha$ introduced in \cite{grohs2023proof} (and described above), represent classes of Deep Operator Networks and kernel-integral based neural operators with a parametric approximation rate $\alpha$. They will be introduced and studied in Sections \ref{sec:don} and \ref{sec:no}. 

In the context we are able to generalize the theory-to-practice gap to the infinite-dimensional setting: Assuming access to an optimal algorithm for the reconstruction of the underlying mapping from $N$ samples, we show that the best achievable convergence rate $N^{-\beta_\ast}$ on the relevant approximation spaces is upper bounded by $1/p$, if the reconstruction error is measured in a Bochner $L^p$-norm. These results are provided in Theorems \ref{thm:main-don} and \ref{thm:unif-don} (Deep Operator Networks) and Theorems \ref{thm:main-no} and \ref{thm:unif-no} (kernel-integral based neural operators) and they uncover fundamental limits on the convergence guarantees that are possible in the context of operator learning: one cannot improve on the standard Monte-Carlo rate, irrespective of how high the parametric convergence rate may be.

In summary, this work makes the following contributions:
\begin{itemize}
\item We give a unified treatment of the theory-to-practice gap in a general $L^p$-setting, valid for arbitrary $p\in [1,\infty]$, sharpening available bounds in the literature. \mt{Specifically, the bounds we achieve in the finite-dimensional setting is sharper than that of \cite{grohs2023proof}.}
\item \mt{Going beyond the setting of \cite{grohs2023proof},} we extend the theory-to-practice gap to the infinite-dimensional setting of operator learning. Our results apply to Deep Operator Networks and integral kernel-based neural operators, including the Fourier neural operator.
\item For these architectures, we show that the optimal convergence rate in a Bochner $L^p$-norm is bounded by $\beta_\ast \le 1/p$, for any $p\in [1,\infty)$.
\item We furthermore show that $\beta_\ast = 0$ for \emph{uniform} approximation over a compact set of input functions, i.e. no algebraic convergence rates are possible.
\end{itemize}

\mt{Ultimately, the results of this work suggest that the practical limitation of learning operators arises from an information-theoretic barrier rather than a model expressivity barrier. In other words, one is limited by the ability to extract information from finite data samples.}

\subsection{Overview}
In Section \ref{sec:gap-finite} we first review neural network approximation spaces and relevant notions from sample complexity theory. In Section \ref{sec:ttp-gap}, we then state and prove our main result in the finite dimensional setting, Theorem \ref{thm:main}. In Section \ref{sec:gap-infinite}, we extend the finite-dimensional result to operator learning. After a brief review of relevant concepts, we discuss the approximation-theoretic setting in Section \ref{sec:input-functions}. In Section \ref{sec:finite-to-infinite} we state an abstract result, Proposition \ref{prop:Lpop}, which establishes a connection between the finite- and infinite-dimensional settings. Finally, based on this abstract result, we derive an infinite-dimensional theory-to-practice gap for approximation in $L^p$- and sup-norms. Section \ref{sec:don} discusses Deep Operator Networks, resulting in Theorems \ref{thm:main-don} and \ref{thm:unif-don}, respectively. Section \ref{sec:no} discusses kernel-integral based neural operators, resulting in Theorems \ref{thm:main-no} and \ref{thm:unif-no}. We end with conclusions and further discussion in Section \ref{sec:conclusion}.

\subsection{Notation}
 For a vector $v\in \R^d$, we indicate by $|v|$ the Euclidean norm, and for a matrix $A\in \R^{m\times d}$, we denote by $\Vert A \Vert = \sup_{|v|=1} |Av|$ its operator norm. Given a domain $D \subset \R^d$, we denote by $W^{1,\infty}(D;\R^m)$ the set of measurable functions $u: D \to \R^m$ with uniformly bounded values and uniformly bounded weak derivatives (Lebesgue almost everywhere). The corresponding norm is defined as 
\begin{align*}
    \|u\|_{W^{1,\infty}(D; \R^m)}= \Vert u \Vert_{L^\infty(D;\R^m)} + |u|_{W^{1,\infty}(D;\R^m)},
\end{align*}
where $|u|_{W^{1,\infty}(D;\R^m)} = \esssup_{x\in D} \Vert Du(x)\Vert$ denotes the $W^{1,\infty}$ seminorm. Given a topological space $\cX$, we will denote by $\cP(\cX)$ the set of all probability measures on $\cX$ under the Borel $\sigma$-algebra. For two measures $\mu$ and $\nu$, on a measure space $(\Omega, \Sigma)$, we will write $\mu \geq \nu$ to indicate that $\mu(B) \geq \nu(B)$ for any element $B \in \Sigma$.

\section{A generalized gap in finite dimension}
\label{sec:gap-finite}

The goal of this section is to derive new lower bounds on the sample complexity of ReLU neural network approximation spaces, in a finite-dimensional setting. To describe the mathematical setting, we recall relevant notions from \cite[Sect. 2.1 and 2.2]{grohs2023proof}, below.

\subsection{ReLU neural networks}
\label{sec:relu}

Following \cite{grohs2023proof}, a \define{neural network} is defined as a tuple $\psi = \left((A_1, b_1), \ldots, (A_L, b_L)\right)$, consisting of matrices $A_j \in \mathbb{R}^{d_j \times d_{j-1}}$ and bias vectors $b_j \in \mathbb{R}^{d_j}$. 
The number of layers $L(\psi) := L$ is the \define{depth} of $\psi$, and $W(\psi) := \sum_{j=1}^{L}\left(\|A_j\|_{\ell^0} + \|b_j\|_{\ell^0}\right)$ is used to denote the \define{number of (nonzero) weights} of $\psi$. Here, the notation $\|A\|_{\ell^0}$ refers to the number of nonzero entries of a matrix (or vector) $A$. Finally, we write $d_{\text{in}}(\psi) := d_0$ and $d_{\text{out}}(\psi) := d_L$ for the \define{input and output dimension} of $\psi$, and we set $\|\psi\|_{\NN} := \max_{j=1,\ldots,L} \max\{\|A_j\|_\infty, \|b_j\|_\infty\}$, where $\|A\|_\infty := \max_{i,j} |A_{i,j}|$.

The function $R_{\sigma} \psi$ computed by $\psi$, is defined via an activation function $\sigma$. In this paper, we restrict attention to ReLU networks, corresponding to $\sigma : \mathbb{R} \rightarrow \mathbb{R},\ x \mapsto \max\{0, x\}$, and acting componentwise on vectors. The \define{realization} $R_{\sigma} \psi : \mathbb{R}^{d_0} \rightarrow \mathbb{R}^{d_L}$ is then given by
\[
R_{\sigma} \psi := T_L \circ (\sigma \circ T_{L-1}) \circ \cdots \circ (\sigma \circ T_1) \quad \text{where} \quad T_j : 
    \left\{\begin{array}{ccc}\R^{d_{j-1}} & \to & \R^{d_j} \\ x & \mapsto & A_j x + b_j\end{array}\right..
\]
With slight abuse of notation, we usually do not distinguish notationally between $\psi$ and its ReLU-realization $R_\sigma\psi$. Thus, we simply say that $\psi: \R^{d_0} \to \R^{d_L}$ is a (ReLU) neural network, when referring to the realization associated with a specific setting of the weights matrices and biases $\psi = \left((A_1, b_1), \ldots, (A_L, b_L)\right)$.

\subsection{Neural network approximation spaces}
\label{sec:approx-space}
We next summarize the relevant approximation spaces $A^{\alpha,\infty}_{\bl}(D)$ for a domain $D\subset \R^d$; they depend on a generalized `smoothness' parameter $\alpha > 0$ and a depth-growth function $\bl = \bl(n)$. In short, $A^{\alpha,\infty}_{\bl}(D)$ contains all functions $f: D \to \R$ that can be uniformly approximated at rate $n^{-\alpha}$, by ReLU neural networks with at most $n$ non-zero weights and biases, depth at most $\bl(n)$, and weight magnitude at most $1$. \mt{The approximation spaces formulated below are similar to the approximation classes described in the seminal work \cite{devore2021neural} in Section 11.2. For the purposes of this work, we take the definitions a step further by putting additional depth and weight magnitude bounds on the ReLU neural network class $\Sigma_{d,n}^\ell$ defined below and defining a quasi-norm on the approximation class.}

More precisely, given an input dimension $d \in \mathbb{N}$ and a non-decreasing depth-growth function $\bl : \mathbb{N} \rightarrow \mathbb{N}\cup \{\infty\}$, we
define
\[
\Sigma_{d,n}^{\bl} := \set{ \psi: \R^d \to \R }{ 
\begin{array}{l}
\psi \text{ NN with } d_{\text{in}}(\psi) = d, d_{\text{out}}(\psi) = 1, \\
W(\psi) \leq n, L(\psi) \leq \bl(n), \|\psi\|_{\NN} \leq 1
\end{array}
}.
\]
Then, given a measurable subset $D \subset \mathbb{R}^{d},\ p \in [1,\infty],$ and $\alpha \in (0,\infty),$ for each
measurable $f : D \rightarrow \mathbb{R}$, we define
\[
\Gamma^{\alpha,p}(f) := \max \left\{ \|f\|_{L^{p}(D)},\ \sup_{n \in \mathbb{N}} \left[ n^{\alpha} \cdot d_{p,D}\left( f,\Sigma_{d,n}^{\bl}\right) \right] \right\} \in [0,\infty],
\]
where $d_{p,D}(f,\Sigma) := \inf_{g \in \Sigma} \|f - g\|_{L^{p}(D)}.$

As pointed out in \cite{grohs2023proof}, $\Gamma^{\alpha,p}$ is not a (quasi-)norm. However, one can define a neural network \define{approximation space quasi-norm} $\| \cdot \|_{A_{\bl}^{\alpha,p}}$ by
\[
\|f\|_{A_{\bl}^{\alpha,p}} := \inf \left\{\theta > 0 : \Gamma^{\alpha,p}(f/\theta) \leq 1\right\} \in [0,\infty],
\]
giving rise to the \define{approximation space}
\[
A_{\bl}^{\alpha,p} := A_{\bl}^{\alpha,p}(D) := \left\{f \in L^{p}(D) : \|f\|_{A_{\bl}^{\alpha,p}} < \infty\right\}.
\]
 We denote by $U^{\alpha,p}_{\bl}(D) \subset A_{\bl}^{\alpha,p}(D)$ the unit ball,
\[
U_{\bl}^{\alpha,p} := U_{\bl}^{\alpha,p}(D) := \left\{f \in L^{p}(D) : \|f\|_{A_{\bl}^{\alpha,p}} \le 1\right\}.
\]
As shown in \cite{grohs2023proof}, $U^{\alpha,p}_{\bl}(D)$ consists precisely of those $f$ for which $d_{p,D}(f,\Sigma_{d,n}^{\bl}) \le n^{-\alpha}$ for all $n\in \N$. We will focus on the special case $p=\infty$, in the following. 

The following quantity is crucial in characterizing the sample complexity of neural network approximation spaces \cite{grohs2023proof}:
\begin{align}
\label{eq:ellbd}
\bl^* := \sup_{n \in \mathbb{N}} \bl(n) \in \mathbb{N} \cup \{\infty\}.
\end{align}
It describes the maximal depth that is allowed for a neural network approximant of a given function.

\begin{remark}
We will later extend the definition of the approximation spaces $A^{\alpha,\infty}_{\bl}$ to the infinite-dimensional operator learning setting. Specifically, in Section \ref{sec:don} we define $\A^{\alpha,\infty}_{\bl,\DON}$ for a class of DeepONets and in Section \ref{sec:no}, we define $\A^{\alpha,\infty}_{\bl,\NO}$ for (integral-kernel based) Neural Operators.
\end{remark}

\subsection{Sample Complexity}
\label{sec:sample-complexity}

By definition, the approximation of a function $f \in A^{\alpha,\infty}_{\bl}(D)$ by neural networks can be achieved at a rate $n^{-\alpha}$, in terms of the required number of neural network parameters. This rate is fast, when $\alpha \gg 1$, thus allowing for efficient approximation in principle. Despite this fact, it has been shown in \cite{grohs2023proof} that such high rates, in terms of the parameter count $n$, do not imply correspondingly high convergence rates when considering the required number of samples to train such neural networks. This  phenomenon is referred to as the \emph{theory-to-practice gap}.

Below, we recall mathematical notions from \cite{grohs2023proof}, which quantify the sample complexity of a set of continuous functions $U \subset C(D)$, where $D \subset \R^d$ is a subdomain. We are mostly interested in the setting where $U = U^{\alpha,\infty}_{\bl}(D)$ is the unit ball in the relevant neural network approximation space.

\subsubsection{The Deterministic Setting}

Given $U\subset C(D)$, we now consider the approximation of $f\in U$ with respect to the $L^p(D)$-norm.

A map $A : U \rightarrow L^p(D)$ is called a \emph{deterministic method using $N \in \mathbb{N}$ point measurements}, if there exists $\mathbf{x} = (x_1, \ldots, x_N) \in D^N$ and a map $Q : \mathbb{R}^N \rightarrow L^p(D)$ such that
\[
A(f) = Q\left(f(x_1), \ldots, f(x_N)\right) \quad \forall \, f \in U.
\]
The set of all deterministic methods using $N$ point measurements will be denoted by $\Alg^{\text{det}}_N(U,L^p(D))$. 

We define the error of $A$ for approximation in $L^p$ as 
\begin{align*}
e(A, U, L^p(D)) := \sup_{f \in U} \|f - A(f) \|_{L^p(D)}.
\end{align*}

The \define{optimal error} for (deterministic) approximation in $L^p$ using $N$ point samples is then
\begin{align*}
e_N^{\text{det}}(U, L^p(D)) := \inf_{A \in \Alg^{\text{det}}_N(U, L^p(D))} e(A, U, L^p(D)).
\end{align*}

Finally, the \define{optimal order of convergence} for (deterministic) approximation in $L^p$ using $N$ point samples is 
\begin{align}
\label{eq:beta-det}
\beta_\ast^{\text{det}}(U, L^p(D)) := 
\sup \left\{ \beta \geq 0 : 
\begin{array}{ll}
\exists \, C > 0 \,\text{s.t. } \forall N \in \mathbb{N},  \\ 
e_N^{\text{det}}(U, L^p(D)) \leq C \cdot N^{-\beta} 
\end{array}
\right\}.
\end{align}

\begin{example}
    \mt{As an example, for a fixed set of points $\{x_1, \dots, x_N\}$, the optimal recovery algorithm is known, although difficult to compute in practice. It is obtained by returning the center of the smallest $L^p$ ball in $U$ containing the set $\{\hat{f}: \; f \in U, \hat{f}(x_i) = f(x_i), \; i \in [N]\}$. This ball is known as the Chebyshev ball \cite{micchelli1977optimal}. Furthermore, returning any function in this set incurs error only up to a factor of $2$ greater than the optimal error. So the deterministic setting decomposes into two parts; first, choosing sample points, and second, computing the center of the Chebyshev ball \cite[Section 10.2]{devore2021neural}.}
\end{example}

\subsubsection{The Randomized Setting} Generalizing the deterministic setting above, we consider randomized algorithms. A \emph{randomized method using $N \in \mathbb{N}$ point measurements (in expectation)} is a tuple $(A, \mathsf{N})$ consisting of a family $A = (A_\omega)_{\omega \in \Omega}$ of maps $A_\omega : U \rightarrow L^p(D)$ indexed by a probability space $(\Omega, \mathcal{F}, \mathbb{P})$ and a measurable function $\mathsf{N} : \Omega \rightarrow \mathbb{N}$ with the following properties:
\begin{enumerate}
    \item for each $f \in U$, the map $\Omega \rightarrow L^p(D), \omega \mapsto A_\omega(f)$ is measurable with respect to the Borel $\sigma$-algebra on $L^p(D)$,
    \item for each $\omega \in \Omega$, we have $A_\omega \in \Alg^{\text{det}}_{\mathsf{N}(\omega)}(U, L^p(D))$,
    \item $\mathbb{E}_\omega[\mathsf{N}(\omega)] \leq N$.
\end{enumerate}
We say that $(A, \mathsf{N})$ is \emph{strongly measurable} if the map $\Omega \times U \rightarrow L^p(D), (\omega, f) \mapsto A_\omega(f)$ is measurable, where $U \subset C(D)$ is equipped with the Borel $\sigma$-algebra induced by the uniform norm. We denote the set of all strongly measurable $(A,\mathsf{N})$ satisfying the above properties by $\Alg_N(U, L^p(D))$.

The \define{expected error} of a randomized algorithm $(A, \mathsf{N})$ for approximation in $L^p$ is defined as
\begin{align}
\label{eq:err-ran1}
e((A, \mathsf{N}), U, L^p(D)) := \sup_{f \in U} \mathbb{E}_\omega \left[\|f - A_\omega(f)\|_{L^p(D)}\right].
\end{align}

The \define{optimal randomized error} for approximation in $L^p$ using $N$ point samples (in expectation) is
\begin{align}
\label{eq:err-ran}
e_N^{\text{ran}}(U, L^p(D)) := \inf_{(A, \mathsf{N}) \in \Alg_N(U, L^p(D))} e((A, \mathsf{N}), U, L^p(D)).
\end{align}

Finally, the \define{optimal randomized order} for approximation in $L^p$ using point samples is
\begin{align}
\label{eq:beta-ran}
\beta_\ast(U, L^p(D)) := \sup \left\{ \beta \geq 0 : 
\begin{array}{ll}
\exists\, C > 0\; \text{s.t. } \forall N \in \mathbb{N}, \\
e_N^{\text{ran}}(U, L^p(D)) \leq C \cdot N^{-\beta} 
\end{array}
\right\}.
\end{align}
We point out that a deterministic method is a special case of a randomized method, and hence
\begin{align}
\beta^\det_\ast(U,L^p(D)) \le \beta_\ast(U,L^p(D)).
\end{align}
Since our aim is to derive upper bounds on these convergence rates, we may restrict attention to $\beta_\ast(U,L^p(D))$, implying corresponding bounds also for $\beta_\ast^\det(U,L^p(D))$.

In the following, we in particular derive upper bounds for the exponents 
$\beta_\ast(U^{\alpha,\infty}_{\bl}, L^p)$,
where $U^{\alpha,\infty}_{\bl} = U^{\alpha,\infty}_{\bl}([0,1]^d)$ and $L^p = L^p([0,1]^d)$ with $1\le p \le \infty$. 

\subsection{Theory-to-practice gap}
\label{sec:ttp-gap}
It was shown in \cite[Theorem 1.1 and 1.2]{grohs2023proof}, that the best possible convergence rate of any reconstruction method on $U^{\alpha,\infty}_{\bl} := U^{\alpha,\infty}_{\bl}([0,1]^d)$ based on point samples is upper bounded as follows. When the reconstruction error is measured with respect to the $L^\infty([0,1]^d)$ norm, the rate is upper bounded by
\begin{align}
\label{eq:bd1}
\beta_\ast(U^{\alpha,\infty}_{\bl},L^\infty)
\le 
\frac 1d \cdot \frac{\alpha}{\alpha+\L}.
\end{align}
When measuring the reconstruction error with respect to the $L^2([0,1]^d)$ norm, the bound becomes
\begin{align}
\label{eq:bd2}
\beta_\ast(U^{\alpha,\infty}_{\bl},L^2) \le \frac 12 + \frac{\alpha}{\alpha+\L}.
\end{align}
These are the tightest known upper bounds on the convergence rates for the most relevant range $\alpha \ge \max(2,\L)$; for smaller values of $\alpha$ refined estimates are given in \cite[Theorems 5.1 and 7.1]{grohs2023proof}. 
Astonishingly, even in the limit $\alpha \to \infty$, the best possible reconstruction rates are upper bounded by $1/d$ and $3/2$, respectively. 

Our first goal in the next subsection is to extend and sharpen the bounds in \eqref{eq:bd1} and \eqref{eq:bd2} for arbitrary $d\in \N$ and $p \in [1,\infty]$, with a unified proof.

\paragraph{Main result in finite dimension}
We prove the following theorem:
\begin{theorem}
\label{thm:main}
Let $p\in [1,\infty]$. Let $\bl: \N \to \N \cup \{\infty\}$ be non-decreasing with $\bl^\ast \ge 3$. Given $d\in \N$ and $\alpha \in (0,\infty)$, consider 
\[
U := U^{\alpha,\infty}_{\bl}(\R^d)|_{[0,1]^d} :=  \set{f|_{[0,1]^d}}{f \in U^{\alpha,\infty}_{\bl}(\R^d)},
\]
such that $U\subset C([0,1]^d)$. Then 
\begin{align}
\label{eq:main}
\beta_\ast(U,L^p([0,1]^d))
\le 
\frac 1p + \frac{1}{d} \cdot \frac{\alpha}{\alpha+\L}.
\end{align}
\end{theorem}

Since the restriction $U = U^{\alpha,\infty}_{\bl}(\R^d)|_{[0,1]^d}$ is a subset of $U^{\alpha,\infty}_{\bl}([0,1]^d)$, Theorem \ref{thm:main} implies the following corollary:

\begin{corollary}
Denote $U^{\alpha,\infty}_{\bl} = U^{\alpha,\infty}_{\bl}([0,1]^d)$. Under the assumptions of Theorem \ref{thm:main}, we have
\[
\beta_\ast(U^{\alpha,\infty}_{\bl}, L^p([0,1]^d))
\le 
\frac 1p + \frac{1}{d} \cdot \frac{\alpha}{\alpha+\L}.
\]
\end{corollary}

\begin{remark}
The upper bound \eqref{eq:main} implies in particular, that the best possible convergence rate of any randomized or deterministic reconstruction method is upper bounded by $\beta\le \frac{1}{p} + \frac{1}{d}$. In high-dimensional applications, this upper bound is approximately $\frac{1}{p}$, implying that algorithms cannot be expected to converge at substantially faster than Monte-Carlo rates. \mt{It is an open problem whether the bound in \eqref{eq:main} is sharp or whether the bound may be tightened further towards $\frac{1}{p}$.}
\end{remark}

\subsection{Proof of Theorem \ref{thm:main}}

Our proof of Theorem \ref{thm:main} is based on several technical lemmas and ideas from \cite{grohs2023proof}. The main novelty here is a different approach to combining these ingredients: While the derivation in \cite{grohs2023proof} relies on a linear combination of (many) hat functions and combinatorial arguments, our proof will be instead be based on a random placement of a single hat function, and a probabilistic argument. \mt{Thus, the novelty of our contribution in the finite-dimensional setting is twofold; (i) the obtained bounds are sharper and (ii) the probabilistic proof is new. In obtaining this result, we have lifted Lemmas \ref{lem:loc} and \ref{lem:rand} with minor modifications from \cite{grohs2023proof}, and Lemmas \ref{lem:void} and \ref{lem:randest} are entirely new.}

\paragraph{Outline}
Before detailing the proof of Theorem \ref{thm:main}, we first outline the general idea on the domain $[0,1]^d$. Our first observation is that, for any choice of evaluation points $x_1, \dots, x_N \in [0,1]^d$, there exists a void with inner diameter of order $N^{-1/d}$; more precisely, we show that, independently of the choice of $x_1,\dots, x_N$, for a randomly drawn $y \sim \unif([0,1]^d)$, we have 
\begin{align}
\label{eq:illustrate}
\min_{j=1,\dots, N} |y-x_j| \ge \frac14 N^{-1/d},
\end{align} 
with probability at least $1/2$.

Given the inevitable presence of such a void, we are then tempted to place a function $g$ with support inside this void. If we assume that $\Vert g \Vert_{L^\infty} \le 1$, or indeed that $g$ is the characteristic function of a cube in this void, then such $g$ can have an $L^p$-norm as large as $\Vert g \Vert_{L^p} \sim N^{-1/p}$. Given only point values at $x_1,\dots, x_N$, such $g$ will be indistinguishable from the zero-function. Thus, if $A$ is a reconstruction method relying only on the point values at $x_1,\dots, x_N$, then $A(g) = A(0)$. It follows that $\Vert g\Vert_{L^p} \le \Vert g - A(g) \Vert_{L^p} + \Vert 0 - A(0)\Vert_{L^p}$ and hence at least one of $\Vert g - A(g) \Vert_{L^p}$ or $\Vert 0 - A(0) \Vert_{L^p}$ must be on the order of magnitude of $\Vert g \Vert_{L^p([0,1]^d)} \sim N^{-1/p}$. As a consequence, the achievable convergence rate $\beta$ of a reconstruction method on characteristic functions is fundamentally limited to $\beta \le \frac1p$. 

To link the above observation with reconstruction methods on neural network approximation spaces, we will recall  (cp. Lemma \ref{lem:loc}, below) that ReLU neural networks can efficiently approximate certain localized functions $\vartheta_{M,y}$. These functions are locally supported inside a cube of side-length $r:=1/M$ with center $y\in [0,1]^d$. In our proof, the $\vartheta_{M,y}$ with $M \sim N^{1/d}$ will act as a replacement of the characteristic function $g$ of the outline, above. The main difference with the characteristic function is that the gradient of $\vartheta_{M,y}$ cannot be arbitrarily large, if we constrain it to belong to the unit ball $U^{\alpha,\infty}_{\bl}$ in the approximation space. This limit on the gradient introduces an additional correction in our upper bound, which finally will take the form $\beta_\ast \le \frac1p + \text{(correction depending on $\alpha, \bl$)}$.

\paragraph{Details}
We now proceed with the detailed proof of Theorem \ref{thm:main}. For any placement of evaluation points $x_1, \dots, x_N \in [0,1]^d$, we first show that a point $y\in [0,1]^d$ picked uniformly at random has a positive chance of sitting in a ``void'' with interior diameter $\sim N^{-1/d}$:
\begin{lemma}[Existence of a void]
\label{lem:void}
Let $d,N\in \N$. Let $x_1, \dots, x_N \in [0,1]^d$ be given. Consider $y \sim \unif([0,1]^d)$ drawn uniformly at random. Then
\begin{align}
\Prob_y\left[
\min_{j=1,\dots, N} |y-x_j|_\infty > \frac14 N^{-1/d}
\right]
\ge \frac12.
\end{align}
\end{lemma}

\begin{proof}
Let us fix $r>0$ for the moment. Applying a union bound, we note that
\begin{align*}
\Prob_y
\left[ 
\min_{j=1,\dots, N} |y - x_j|_\infty \le r
\right]
&= 
\Prob_y
\left[ 
y \in \textstyle\bigcup_{j=1}^N \left( x_j + [-r,r]^d \right)
\right]
\\
&\le 
\sum_{j=1}^N \mathrm{Vol}\left(x_j + [-r,r]^d\right)
= N (2r)^d.
\end{align*}
It thus follows that, if $r := \frac14 N^{-1/d} \le \frac12(2N)^{-1/d}$, then 
\[
\Prob_y
\left[ 
\min_{j=1,\dots, N} |y - x_j|_\infty \le r
\right]
\le 
\frac12.
\]
Hence, we must have $\Prob_y
\left[ 
\min_{j=1,\dots, N} |y - x_j|_\infty > r
\right]
\ge 
\frac12$, as claimed.
\end{proof}

We now state the following fundamental result, which follows from \cite[Lemma 3.4]{grohs2023proof}. The main insight of this lemma is that ReLU neural networks can efficiently approximate a localized function $\vartheta_{M,y}: \R^d \to [0,1]$, which is supported in a cube of side-length $2/M$ around $y$, and which ``fills in'' a significant fraction of this cube.
\begin{lemma}[Localized networks]
\label{lem:loc}
Given $d\in \N$, $M\ge 1$ and $y\in [0,1]^d$, there exists a function $\vartheta_{M,y}: \R^d \to [0,1]$ with the following properties:
\begin{itemize}
\item $\vartheta_{M,y}(x)$ depends continuously on $x$ and $y$; in fact, we have that $\vartheta_{M,y}(x) = \vartheta_M(x-y)$ is a shift of a neural network $\vartheta_M$.
\item $\vartheta_{M,y}(x) = 0$ whenever $|x-y|_\infty \ge 1/M$,
\item For any $p\in [1,\infty]$, there exists $C = C(d,p)>0$ satisfying,
\[
\Vert \vartheta_{M,y}\Vert_{L^p([0,1]^d)} \ge C M^{-d/p},
\]
\item There exists a constant $\kappa = \kappa(\gamma, \alpha, d, \bl, c) > 0$, such that 
\[
g_{M,y} := \kappa M^{-\alpha/(\alpha+\L)} \vartheta_{M,y} \in U^{\alpha,\infty}_{\bl}(\R^d).
\]
\end{itemize}
\end{lemma}

\begin{proof}
The existence of $\vartheta_{M,y}$ and the other properties are shown in \cite{grohs2023proof}, see in particular \cite[Lemma 3.4]{grohs2023proof}.
\end{proof}

We also state the following result, which is a minor variant of \cite[Lemma 2.3]{grohs2023proof}:
\begin{lemma}
\label{lem:rand}
Let $D\subset \R^d$ and $\emptyset \ne U \subset C(D)$ be bounded, and let $1\le p \le \infty$. Assume that there exists $\lambda \in [0,\infty)$, $\kappa > 0$, such that for every $N \in \N$, there exists a probability space $(\Xi,\P)$ and a random variable $\psi: \Xi \to U$, $\xi \mapsto \psi_\xi$, such that
\[
\E_\xi \left[ 
\Vert \psi_\xi - A(\psi_\xi) \Vert_{L^p(D)}
\right]
\ge \kappa N^{-\lambda},
\quad
\forall \, A \in \mathrm{Alg}_N^{\mathrm{det}}(U,L^p(D)).
\]
Then $\beta_\ast(U,L^p(D)) \le \lambda$.
\end{lemma}
\begin{proof}
This follows by the same reasoning as \cite[Lemma 2.3]{grohs2023proof}.
\end{proof}
Our main interest in this result is when $U = U^{\alpha,\infty}_{\bl}(\R^d)|_{[0,1]^d}$. Combining the results from Lemmas \ref{lem:void}, \ref{lem:loc} and \ref{lem:rand}, we now come to the proof of Theorem \ref{thm:main}.

\begin{proof}[Proof of Theorem \ref{thm:main}]

Let $d,N\in \N$ be given. Recall that $U := U^{\alpha,\infty}_{\bl}(\R^d)|_{[0,1]^d}$. Our goal is to apply Lemma \ref{lem:rand} to deterministic reconstruction methods based on $N$ point-values,
\[
A \in \mathrm{Alg}_N^{\mathrm{det}}(U, L^p([0,1]^d)).
\]
To construct a suitable probability space $(\Xi,\P)$ and random function $\Xi \to U$, we first define a probability space as $\Xi := [0,1]^d \times \{-1,+1\}$ endowed with the Borel $\sigma$-algebra and define $\P := \Unif([0,1]^d)\otimes \Unif(\{-1,+1\})$. We will denote elements of $\Xi$ as $\xi = (y,\sigma)$. This defines our probability space. 

We next fix 
\begin{align}
\label{eq:M}
M := 4 N^{1/d}.
\end{align}
With this choice of $M$, we then define $\psi_\xi := \psi_{(y,\sigma)} := \sigma g_{M,y} \in U$, where $g_{M,y}$ is the function of Lemma \ref{lem:loc}. Thus, $\psi_{\xi} = \psi_{(y,\sigma)}$ is a random function, given by a random shift of a localized neural network $\vartheta_M$ by $y\sim \Unif([0,1]^d)$ and a random choice of sign $\sigma \sim \Unif(\{-1,+1\})$. This defines our random variable 
\begin{align}
\label{eq:rv}
\psi: \Xi \to U.
\end{align}

Invoking Lemma \ref{lem:rand}, it will suffice to prove the following claim, formulated as a lemma for later reference:
\begin{lemma}
\label{lem:randest}
Let $\psi: \Xi \to U$, $\xi \mapsto \psi_\xi$ be the random variable defined above \eqref{eq:rv}, where $U:= U^{\alpha,\infty}_{\bl}(\R^d)|_{[0,1]^d}$. There exists a constant $\kappa$, independent of $N$, such that
\begin{align}
\label{eq:ys}
\E_{\xi} 
\left[
\Vert \psi_\xi - A(\psi_\xi) \Vert_{L^p([0,1]^d)}
\right]
\ge \kappa N^{-\lambda},
\end{align}
for all \mt{$A \in \Alg_N^{\det}(U,L^p([0,1]^d))$} and where $\lambda := \frac1p + \frac{1}{d}\cdot \frac{\alpha}{\alpha+\L}$.
\end{lemma}

The claim of Theorem \ref{thm:main} is then immediate from Lemma \ref{lem:rand} and \ref{lem:randest}.\\

\noindent\textit{Proof of Lemma \ref{lem:randest}.} 
Let \mt{$A\in \Alg_N^\det(U^{\alpha,\infty}_{\bl}(\R^d),L^p([0,1]^d))$} be given. Let $Q: \R^N \to L^p([0,1]^d)$ be the reconstruction mapping associated with $A$ and let $x_1,\dots, x_N \in \R^d$ denote the associated evaluation points, such that $A(f) = Q(f(x_1), \dots, f(x_N))$. We first observe that -- whatever the choice of $x_1, \dots, x_N$ -- we have, by Lemma \ref{lem:void} and our choice of $M = 4N^{1/d}$ in \eqref{eq:M}, 
\[
\Prob_y
\left[ 
\min_{j=1,\dots, N} |y - x_j|_\infty > 1/M
\right]
\ge \frac12.
\]
Now consider the random event 
\begin{align*}
E 
&= 
\set{(y,\sigma) \in \Xi}{\psi_{(y,\sigma)}(x_1) = \dots = \psi_{(y,\sigma)}(x_N) = 0} 
\\
&=  \set{(y,\sigma) \in \Xi}{\sigma g_{M,y}(x_1)=\dots = \sigma g_{M,y}(x_N) = 0},
\end{align*}
where the randomness is introduced by $y\sim \unif([0,1]^d)$ and $\sigma \sim \unif\{-1,+1\}$. Since $g_{M,y}$ is supported in the shifted cube $(y+[-1/M,1/M]^d)$ (cp. Lemma \ref{lem:loc}), it follows that 
\[
\Prob_{y,\sigma}[E] 
\ge \Prob_y
\left[ 
\min_{j=1,\dots, N} |y - x_j|_\infty > 1/M
\right]
\ge \frac12.
\]
We can now finish the proof. To this end, we first observe that 
\begin{align}
\E_{\xi} 
&\left[
\Vert \psi_\xi - A(\psi_\xi) \Vert_{L^p([0,1]^d)}
\right]
\notag
\\
&\qquad \ge 
\E_{\xi} 
\left[
\Vert \psi_\xi - A(\psi_\xi) \Vert_{L^p([0,1]^d)}
\; | \; E
\right] \, \Prob_{\xi}[E]
\notag
\\
&\qquad \ge 
\frac12 \E_{\xi} 
\left[
\Vert \psi_\xi - A(\psi_\xi) \Vert_{L^p([0,1]^d)}
\; | \; E
\right]
\notag
\\
&\qquad = 
\frac12 \E_{\xi} 
\left[
\Vert \psi_\xi - A(0) \Vert_{L^p([0,1]^d)}
\; | \; E
\right].
\label{eq:Q1}
\end{align}
We next note that the random variable $\psi_\xi = \sigma g_{M,y}$, conditioned on $E$, has the same distribution as $-\psi_\xi = -\sigma g_{M,y}$ conditioned on $E$; this follows from the fact that $E$ and $\P$ are invariant under replacement $\sigma \to -\sigma$. Furthermore, it follows from Lemma \ref{lem:loc} that there exists a constant $C = C(\alpha, d,p, \bl)>0$, such that 
\begin{align}
\label{eq:Q2}
\Vert \psi_\xi \Vert_{L^p([0,1]^d)}
=
\Vert g_{M,y} \Vert_{L^p([0,1]^d)}
\ge C M^{-d/p-\alpha/(\alpha+\L)},
\end{align}
for all $M \ge 1$ and $y\in [0,1]^d$.
Hence, it follows with 
\[
\lambda := \frac{1}{p} + \frac{1}{d}\cdot\frac{\alpha}{\alpha + \L},
\]
that
\begin{alignat*}{3}
2C M^{-d\lambda} 
&\le 
 2\E_{\xi} 
\left[
\Vert \psi_\xi \Vert_{L^p([0,1]^d)}
\; | \; E
\right]
&& \text{(by \eqref{eq:Q2})}
\\
&=
\E_{\xi} 
\left[
\Vert \psi_\xi - (-\psi_\xi) \Vert_{L^p([0,1]^d)}
\; | \; E
\right]
\\
&\le 
\E_{\xi} 
\left[
\Vert\psi_\xi - A(0) \Vert_{L^p([0,1]^d)}
\; | \; E
\right]
\\
&\qquad 
+
\E_{\xi} 
\left[
\Vert -\psi_\xi - A(0) \Vert_{L^p([0,1]^d)}
\; | \; E
\right]
\quad
&& \text{(triangle ineq.)}
\\
&= 
2\E_{\xi} 
\left[
\Vert \psi_\xi - A(0) \Vert_{L^p([0,1]^d)}
\; | \; E
\right]
&& \text{(invar. $\sigma \to -\sigma$)}
\\
&\le 
4\E_{\xi} 
\left[
\Vert \psi_\xi - A(\psi_\xi) \Vert_{L^p([0,1]^d)}
\right].
&& \text{(by \eqref{eq:Q1})}
\end{alignat*}
 Thus, we have shown that 
\begin{align*}
\E_{\xi} 
\left[
\Vert \psi_\xi - A(\psi_\xi) \Vert_{L^p([0,1]^d)}
\right]
\ge 
\frac{C}{2} M^{-d\lambda} 
.
\end{align*}
This implies a bound of the desired form \eqref{eq:ys}, since 
\[
\frac{C}{2} M^{-d\lambda} 
=
\frac{C}{2} \left( 4 N^{1/d} \right)^{-d\lambda}
=: \kappa
N^{-\lambda},
\]
where we recall \mt{$\lambda = \frac{1}{p}+\frac{1}{d} \cdot\frac{\alpha}{\alpha+\L}$,} and 
where the constant $\kappa$ depends on $\alpha, d,p, \bl$, but is independent of $N$. This is \eqref{eq:ys} and concludes our proof.
\end{proof}

\section{Extension to operator Learning}
\label{sec:gap-infinite}

We now consider the extension of the theory-to-practice gap to operator learning, i.e. the data-driven approximation of operators mapping between infinite-dimensional function spaces. The finite-dimensional upper bound \eqref{eq:main} suggests that in the infinite-dimensional limit, $d\to \infty$, the best convergence rate in $L^p$ should be upper bounded by $\frac1p$, \emph{independently of $\alpha$ and $\L$}.

Our goal in this second half of the paper is to rigorously state and prove such a theory-to-practice gap for two prototypical classes of neural operators: deep operator networks, as discussed in \cite{lu2021learning,lanthaler2022error}, and a class of integral-kernel based neural operators \cite{kovachki2023neural,LLS2024}. Before stating our main results in the operator-learning setting, we will first summarize the overall objective of operator learning, and discuss suitable infinite-dimensional replacements of the spaces $L^p([0,1]^d)$, for $1\le p \le \infty$ in this context. This is followed by a definition of the aforementioned prototypical neural operator frameworks, their associated operator approximation spaces, and the statement of our main results, establishing a theory-to-practice gap in infinite dimensions.

\subsection{Operator Learning}
\label{sec:operator-learning}
In the following, we will be interested in the sample complexity of operator learning. To simplify our discussion, we will only consider the special case $\cY = \R$, i.e. we consider the data-driven approximation of non-linear functionals $\cG: \cX \to \R$. Since our analysis concerns \emph{lower bounds} on the sample complexity (or \emph{upper bounds} on the optimal convergence rates), the results continue to hold if the output space is replaced by a more general $\cY$ (finite-dimensional or infinite-dimensional). Thus, this reduction can be made without loss of generality. For our analysis, only the fact that the input space $\cX$ is infinite-dimensional will be relevant. 

\subsubsection{Sample Complexity of Operator Learning}

With this operator learning setting in mind, we first point out that the discussion of the sample complexity of Section \ref{sec:sample-complexity} carries over with only minor changes. For convenience, we repeat the main elements here.

We now consider a set of continuous operators $\U \subset C(\cX)= C(\cX;\R)$ on a separable Banach space $\cX$. We fix a Banach space of operators $\V$, equipped with a norm $\Vert \slot \Vert_\V$. In analogy with the finite-dimensional setting, a map $\cA: \U \to \V$ will be called a deterministic method using $N\in \N$ point measurements, if there exists $\mathbf{u} = (u_1,\dots, u_N) \in \cX^N$, and a map $\cQ: \R^N \to \V$, such that 
\[
\cA(\cG) = \cQ(\cG(u_1),\dots, \cG(u_N)), \quad \forall \, \cG \in \U.
\]
We recall that each $\cG \in \U$ is a continuous operator $\cG: \cX \to \R$, and hence the above expression is well-defined.  Consistent with our earlier discussion, the set of all deterministic methods using $N$ point measurements will be denoted by $\Alg^\det_N(\U,\V)$. It will be assumed that $\U \subset \V$ with a canonical embedding.

The approximation error of $\cA$ in $\V$ is defined as
\[
e(\cA,\U,\V) = \sup_{\cG \in \U} \Vert \cA(\cG) - \cG \Vert_\V,
\]
and the optimal error is $e^{\mathrm{det}}_N(\U,\V) = \inf_{\cA \in \Alg_N^{\mathrm{det}}(\U,\V)} e(\cA,\U,\V)$. The optimal order of convergence for deterministic approximation in $\V$ using $N$ point samples is defined as in \eqref{eq:beta-det}. 

Randomized methods for operator learning, as well as the corresponding errors and the optimal randomized order \eqref{eq:beta-ran} are similarly defined, in analogy with Section \ref{sec:sample-complexity}; we here only recall that a randomized method using $N$ point measurement (in expectation) is a tuple $(\cA,\mathsf{N})$ consisting of a family $\cA = (\cA_\omega), {\omega\in \Omega}$ of maps $\cA_\omega: \U \to \V$ indexed by a probability space $(\Omega, \cF, \P)$ and a measurable function $\mathsf{N}:\Omega \to \N$ with the same properties as described there, and with $\cA_\omega \in \Alg^\det_{\mathsf{N}(\omega)}(\U,\V)$. We continue to denote by $\Alg_N(\U,\V)$ the set of all strongly measurable randomized methods with $\E_\omega[\mathsf{N}(\omega)]\le N$ expected point evaluations.

In the following, we will derive upper bounds on $\beta_\ast(\U,\V)$, where $\U$ is the unit ball in an operator learning approximation space, and $\V$ is a space of operators with distance measured by either an $L^p$-norm with $1\le p<\infty$, or by a sup-norm. We describe the relevant setting in the next section.

\subsection{Input functions in infinite dimensions}
\label{sec:input-functions}

In the finite-dimensional case, the domain $[0,1]^d$ is widely considered as a canonical prototype, and results obtained for $[0,1]^d$ usually extend to more general bounded domains $D \subset \R^d$, under mild assumptions. In the infinite-dimensional setting, there is no longer such a canonical choice. Therefore, we preface our discussion of the infinite-dimensional theory-to-practice gap with the introduction of suitable alternatives to the spaces $L^p([0,1]^d)$ and $C([0,1]^d)$ (the latter corresponding to the limit case $p=\infty$) to be considered in this work. 






\paragraph{Approximation in $\V = L^p(\mu)$}

When the approximation error is measured in an $L^p$-norm for $1\le p< \infty$, we consider the following prototypical setting:  We are given a probability measure $\mu$ on the input function space $\cX$. We then consider the Banach space of $p$-integrable (real-valued) operators $\V = L^p(\mu)$, with the following $L^p$-norm:
\[
\Vert \cG \Vert_{L^p(\mu)} := \E_{u\sim \mu}\left[ |\cG(u)|^p \right]^{1/p}.
\]
Thus, in the infinite-dimensional setting, the probability measure $\mu$ replaces the Lebesgue measure on $[0,1]^d$. \mt{In order to obtain the hardness result, we must ensure that $\mu$ is ``truly infinite-dimensional''.} We thus assume that the following $d$-dependent assumption holds for any $d\in \N$:

\renewcommand{\theassumption}{\arabic{section}.\arabic{assumption}\textsuperscript{(d)}}
\begin{assumption} \label{ass:mud}
There exist linearly independent $\{e_1,\dots, e_d\}\subset \cX$ with bi-orthogonal $\{e^\ast_1,\dots,e^\ast_d\}\subset \cX^\ast$, such that $\mu \in \cP(\cX)$ is the law of 
\begin{align}\label{eq:muparam}
u = \xi + \sum_{j=1}^d y_j e_j, \quad (y_1,\dots, y_d)\sim \mu_d, \; \xi \sim \mu_d^\perp,
\end{align}
where
\begin{enumerate}[label=\textup{(\alph*)}]
\item The law of $(y_1,\dots y_d)$ is of the form $\mu_d = \prod_{j=1}^d \rho_j(z_j) \dz_j$, for $\rho_j \in L^1(\R)$.
\item \mt{The law of $\xi$ is of the form $\mu_d^\perp \in \cP(\Omega_d)$ for $\Omega_d := \set{\xi \in \cX}{e^\ast_1(\xi) = \dots = e^\ast_d(\xi) = 0} \subset \cX.$}
\item There exist intervals $I_j\subset \R$ such that $\essinf_{z_j\in I_j} \rho_j(z_j) > 0$ for $j=1,\dots, d$.
\end{enumerate}
\end{assumption}
\renewcommand{\theassumption}{\arabic{section}.\arabic{assumption}}

\mt{We can thus think of $\mu$ in Assumption \ref{ass:mud} as essentially equivalent to the product measure $\mu_d \otimes \mu_d^\perp$. This structure will be convenient for our derivation of the infinite-dimensional theory-to-practice gap. Below we provide an illustrative example of a basic setting where Assumptions \ref{ass:mud} hold.}
\mt{
\begin{example}
Consider the space $\cX = L^2([0,1])$ and the Dirichlet Laplacian sine basis $e_j(x) = \sqrt{2}\sin(j\pi x)$ for $j \in \{1, 2, \dots\}$ and $x \in [0,1]$. Then we have that $e_j^*(u) = \int_0^1 u(x) e_j(x) \dx$. It is straightforward to check that biorthogonality and linear independence hold. Draw independent $Z_j \sim N(0, \sigma_j^2)$ with $\sum_j \sigma_j^2 < \infty$, for instance, with $\sigma_j = \frac{1}{j}$, and set $u = \sum_{j=1}^\infty Z_je_j$. Then $\mu$, the law of $u$, is the trace-class Gaussian measure $N(0, C)$ on $L^2([0,1])$ with covariance $C = \sum_{j=1}^\infty \sigma_j^2 e_j \otimes e_j$. Splitting at $d$ with $y_j = Z_j$ for $j \leq d$ and $\xi = \sum_{j > d} Z_j e_j$ yields the form \eqref{eq:muparam}, where independence of the $Z_j$ gives $\xi \perp (y_1, \dots, y_d)$. The law of $(y_1, \dots, y_d)$ is $\mu_d = N(0, \Sigma_d)$ where $\Sigma_d = \mathrm{diag}(\sigma_1^2, \dots, \sigma_d^2)$ with product density $\mu_d = \prod_{j=1}^d \rho_j(z_j) \mathsf{d}z_j$ for $\rho_j$ the $ N(0, \frac{1}{j^2})$ density. Each Gaussian density satisfies $\essinf_{z \in I_j} \rho_j > 0$ on any bounded interval $I_j$. The law of $\xi$ is $\mu_d^\perp = N(0, C_d^\perp)$ with $C_d^\perp = \sum_{j > d} \sigma_j^2 e_j \otimes e_j$, supported on $\Omega_d$ since $e_k^*(\xi) = 0 $ for $k \leq d$.
\end{example}
}
This concludes our discussion of the $L^p(\mu)$-setting considered in this work.

\paragraph{Approximation in $\V = C(\cK)$} 
When the goal is to approximate $\cG \in \U$ uniformly, i.e. with respect to the $L^\infty$-norm, we restrict attention to a compact set $\cK \subset \cX$. In this case, we consider the Banach space of continuous operators $\V = C(\cK)$, endowed with the sup-norm:
\[
\Vert \cG \Vert_{C(\cK)} := \sup_{u \in \cK} |\cG(u)|.
\]
Thus, in the infinite-dimensional setting, the compact set $\cK$ replaces the unit cube $[0,1]^d$. To ensure that $\cK$ is truly infinite-dimensional, and we make the following assumption:
\begin{assumption}
\label{ass:K}
We assume that the set $\cK$ is (i) \emph{convex} and (ii) there does \emph{not} exist a finite-dimensional subspace $\cX_0 \subset \cX$ containing $\cK$.
\end{assumption}

We can relate this uniform setting to the $L^p(\mu)$-setting described above:
\begin{proposition}
\label{prop:reduction}
Let $\cK$ satisfy Assumption \ref{ass:K}. Then for any $d\in \N$, there exists a probability measure $\mu \in \cP(\cX)$, with $\supp(\mu) \subset \cK$, satisfying Assumption \ref{ass:mud}; more specifically, $\mu$ is the law of 
\[
u = e_0 + \sum_{j=1}^d y_j e_j, \qquad (y_1,\dots, y_d) \sim \prod_{j=1}^d \Unif([0,1]).
\]
Here $e_1, \dots, e_d \in \cX$ are linearly independent elements for which there exist bi-orthogonal elements $e^\ast_1,\dots, e^\ast_d \in \cX^\ast$, and $e_0\in \cK$ is such that $e^\ast_1(e_0) = \dots = e^\ast_d(e_0) = 0$.
\end{proposition}

\begin{proof}
Let $d\in \N$ be given. Our aim is to construct $\mu\in \cP(\cX)$ as in the claim of Proposition \ref{prop:reduction}. By assumption, $\cK$ is infinite-dimensional. It follows that there exist $v_0,\dots, v_d\in \cK$ which are linearly independent. Given such a choice, we set $e_0 := v_0$, and $e_j := \frac1d (v_j - v_0)$ for $j=1,\dots, d$. Since $v_0,\dots, v_d$ are linearly independent, it follows that also $e_0,\dots, e_d$ are linearly independent. Furthermore, for any finite set of linearly independent vectors, we can find bi-orthogonal elements $e^\ast_0,\dots, e^\ast_d\in \cX^\ast$, such that $e^\ast_j(e_k) = \delta_{jk}$, $j,k=0,\dots, d$. In particular, for this choice we then have $e^\ast_1(e_0) = \dots = e^\ast_d(e_0) = 0$.

We now define $\mu\in \cP(\cX)$ as the law of $u := e_0 + \sum_{j=1}^d y_j e_j$, with $y_1,\dots, y_d \overset{iid}{\sim} \Unif([0,1])$. This $\mu$ trivially satisfies Assumption \ref{ass:mud}. To prove Proposition \ref{prop:reduction}, it thus only remains to show that $\supp(\mu) \subset \cK$. To see this, we observe that
\begin{align*}
u 
&= e_0 + \sum_{j=1}^d y_j e_j 
= v_0 + \frac1d \sum_{j=1}^d y_j (v_j - v_0)
\\
&= \left(1-\frac1d \sum_{j=1}^d y_j\right) v_0 + \sum_{j=1}^d \frac{y_j}{d} v_j
=: \sum_{j=0}^d \lambda_j v_j,
\end{align*}
where $\lambda_0 := 1 - \frac1d \sum_{j=1}^d y_j$, $\lambda_j := \frac1d y_j$ for $j=1,\dots, d$.
The last sum identifies $u$ as a convex combination of $v_0,\dots, v_d\in \cK$: Indeed, since $y_j\in [0,1]$, for $j=1,\dots,d$, it follows that $\lambda_0,\dots, \lambda_d\ge 0$ and we also have $\sum_{j=0}^d \lambda_j = 1$ by definition of $\lambda_0$. It now follows from the convexity of $\cK$ that $u\in \cK$ for any choice of $y_1,\dots, y_d \in [0,1]$, and hence $\supp(\mu) \subset \cK$. This concludes the proof.
\end{proof}

A consequence of Proposition \ref{prop:reduction} is that approximation of a continuous operator $\cG:\cK \to \cY$ with respect to the $C(\cK)$-norm is at least as difficult as approximation with respect to the $L^p(\mu)$-norm, for \emph{any} $p \in [1,\infty)$, and where $\mu$ satisfies Assumption \ref{ass:mud}. Indeed, for any reconstruction method $\cA: \U \to C(\cK)$, we have
\[
\Vert \cG - \cA(\cG) \Vert_{L^p(\mu)} 
\le 
\Vert \cG - \cA(\cG) \Vert_{L^\infty(\mu)} 
\le 
\Vert \cG - \cA(\cG) \Vert_{C(\cK)},
\]
where the last inequality follows from $\supp(\mu) \subset \cK$.
This implies that the optimal convergence rate with respect to $C(\cK)$ can be at most as large as the optimal convergence rate with respect to $L^p(\mu)$:
\begin{lemma}
\label{lem:reduction-ran}
Let $\cK\subset \cX$ be compact. If $\mu\in \cP(\cX)$ is a probability measure such that $\supp(\mu)\subset\cK$, then 
\begin{align}
\label{eq:reduction-ran}
\beta_\ast(\U,C(\cK)) \le \beta_\ast(\U,L^p(\mu)).
\end{align}
\end{lemma}
This simple observation will allow us to deduce results about the uniform setting from corresponding results in the $L^p(\mu)$ setting, and derive estimates on $\beta_\ast(\U,C(\cK))$ by passing to the limit $p\to \infty$.

\subsection{From Finite to Infinite Dimensions}
\label{sec:finite-to-infinite}

Neural operators $\Psi: \cX \to \R$, approximating $\cG: \cX \to \R$, can often be interpreted as a composition of two mappings, $\Psi(u) = \psi \circ \cE(u)$. Here, $\cE: \cX \to \R^d$ is an encoder, which maps the infinite-dimensional input into a finite-dimensional latent space, and $\psi: \R^d \to \R$ is the realization of a finite-dimensional neural network. The latent dimension $d\in \N$ is a hyperparameter of the architecture.

The general idea behind our proof of the infinite-dimensional theory-to-practice gap is the following: If the relevant approximation space $\U \subset C(\cX)$ contains compositions of the form $f \circ \cE: \cX \to \R$ where $f \in U^{\alpha,\infty}_{\bl}$ belongs to the $d$-dimensional neural network approximation space, then the mapping 
\[
U^{\alpha,\infty}_{\bl} \to \U, \quad f \mapsto f \circ \cE,
\]
defines an embedding of $U^{\alpha,\infty}_{\bl}$ into $\U$. Thus, the sample complexity of $\U$ should be at least as large as that of $U^{\alpha,\infty}_{\bl}$. As a consequence, any $d$-dimensional upper bound on convergence rate implies a corresponding bound in the infinite-dimensional case. This general intuition is confirmed and made precise by the following proposition:
\begin{proposition}[$L^p$ setting]
\label{prop:Lpop}
Let $\cX$ be a separable Banach space, let $\mu \in \cP(\cX)$ be a probability measure on $\cX$, and let $\U \subset C(\cX)$ be a set of continuous operators. Assume that there is an encoder $\cE: \cX \to \R^d$ and a constant $c>0$, such that 
\begin{align}
\label{eq:c}
\cE_\# \mu \ge c \cdot \unif([0,1]^d).
\end{align}
If for some $\alpha >0$ and $\bl: \N \to \N$, we have
\begin{align}
\label{eq:incl}
\set{
f \circ \cE: \cX \to \R
}{
f \in U^{\alpha,\infty}_{\bl}(\R^d)
}
\subset 
\U,
\end{align}
then 
\begin{align}
\label{eq:Lpop}
\beta_\ast(\U, L^p(\mu))
\le 
\frac{1}{p} + \frac{1}{d} \cdot \frac{\alpha}{\alpha + \L}.
\end{align}
\end{proposition}

Our proof of Proposition \ref{prop:Lpop} will be based on the following two lemmas. The first lemma shows that, under the assumptions of Proposition \ref{prop:Lpop} any deterministic reconstruction method $\cA \in \Alg_N^\det(\U,L^p(\mu))$ in infinite dimensions induces an associated finite-dimensional reconstruction method $A\in \Alg_N^\det(U^{\alpha,\infty}_{\bl},L^p([0,1]^d))$.

\begin{lemma}
\label{lem:Lpop1}
With the notation and under assumptions \eqref{eq:c} and \eqref{eq:incl} of Proposition \ref{prop:Lpop}, the following holds: For any $\cA \in \Alg_N^\det(\U,L^p(\mu))$, there exists $A \in \Alg_N^\det(U^{\alpha,\infty}_{\bl}, L^p([0,1]^d))$, such that 
\begin{align}
\label{eq:lemLpop1}
\Vert f\circ \cE - \cA(f\circ \cE) \Vert_{L^p(\mu)}
\ge 
c 
\Vert f - A(f) \Vert_{L^p([0,1]^d)},
\quad \forall \, f \in U^{\alpha,\infty}_{\bl}.
\end{align}
Here $c>0$ is the constant of \eqref{eq:c}.
\end{lemma}

\begin{proof}[Proof Sketch] 
The detailed proof of Lemma \ref{lem:Lpop1} is included in Appendix \ref{app:Lpop}. The basic idea of the proof is the following: By definition, $\cA$ is of the form $\cA(\cG) = \cQ(\cG(u_1),\dots, \cG(u_N))$ for some sample points $u_1,\dots, u_N\in \cX$ and $\cQ: \R^N\to L^p(\mu)$. We want to construct $A: U^{\alpha,\infty}_{\bl}\to L^p([0,1]^d)$, of the form $A(f) = Q(f(x_1),\dots, f(x_N))$, where $Q: \R^N \to L^p([0,1]^d)$. The canonical choice of the sample points $x_1,\dots, x_N \in \R^d$ is via composition with the encoder, $x_j := \cE(u_j)$. The main remaining question is then how to construct the mapping $Q: \R^N \to L^p([0,1]^d)$ from $\cQ: \R^N \to L^p(\mu)$. A first idea is that this reconstruction could satisfy,
\[
Q(y_1,\dots,y_N)(\cE(u)) := \cQ(y_1,\dots,y_N)(u), \quad \forall \, u \in \cX, 
\; \forall \, (y_1,\dots, y_N) \in \R^N.
\]
However, this is not well-defined, since different $u$ will generally map to the same $x=\cE(u)$. The improved guess is the following: Fix $x\in \R^d$ and consider a random variable $u\sim \mu$. We now condition on the event that $\cE(u)=x$. This gives a conditional distribution on the input function space. We then \emph{average} the reconstruction $\cQ(y_1,\dots,y_N)(u)$ in $u$ over this conditional distribution, i.e. define
\[
Q(y_1,\dots, y_N)(x) := \E_{u\sim \mu}\left[  \cQ(y_1,\dots,y_N)(u) \cond \cE(u)=x \right].
\]
This is well-defined, and due to Jensen's inequality the conditional averaging on the right-hand side turns out to reduce the reconstruction error of $A$ compared to $\cA$. The detailed calculations are provided in Appendix \ref{app:Lpop}.
\end{proof}
The last lemma is in anticipation of our next result, Lemma \ref{lem:Lpop2}. The final result Proposition \ref{prop:Lpop} will then be an immediate consequence.

\begin{lemma}
\label{lem:Lpop2} 
With the notation and under assumptions \eqref{eq:c} and \eqref{eq:incl} of Proposition \ref{prop:Lpop}, the following holds.
There exists $\kappa > 0$, such that for every $N\in \N$, there exists a probability space $(\Xi,\P)$ and a random variable $\Psi: \Xi \to \U$, $\xi \mapsto \Psi_\xi$, such that
\begin{align}
\label{eq:lemLpop2}
\E_\xi\left[ 
\Vert \Psi_\xi - \cA(\Psi_\xi) \Vert_{L^p(\mu)}
\right]
\ge \kappa N^{-\lambda}, \quad \forall \, \cA \in \Alg_N^{\det}(\U,L^p(\mu)),
\end{align}
where $\lambda = \frac{1}{p} + \frac{1}{d} \cdot \frac{\alpha}{\alpha+\L}$.
\end{lemma}

\begin{proof}
By Lemma \ref{lem:randest}, there exists a constant $\kappa > 0$, such that for any $N\in \N$, there exists a probability space $(\Xi,\P)$ and random variable $\psi: \Xi \to U^{\alpha,\infty}_{\bl}$, $\xi \mapsto \psi_\xi$, such that 
\[
\E_\xi\left[ 
\Vert \psi_\xi - A(\psi_\xi) \Vert_{L^p([0,1]^d)}
\right] 
\ge \kappa N^{-\lambda}, 
\quad \forall \, A \in \Alg_N^\det(U^{\alpha,\infty}_{\bl},L^p([0,1]^d)), 
\]
and where $\lambda := \frac{1}{p} + \frac{1}{d}\cdot \frac{\alpha}{\alpha+\L}$. We use $\psi_\xi$ to define a new random variable $\Psi: \Xi \to \U$, $\Psi_\xi := \psi_\xi \circ \cE$. We claim that \eqref{eq:lemLpop2} holds for this $\Psi_\xi$.

To see this, let $\cA \in \Alg_N^\det(\U,L^p(\mu))$ be given. By Lemma \ref{lem:Lpop1}, there exists $A \in \Alg_N^\det(\U,L^p([0,1]^d))$, such that 
\[
\Vert \psi_\xi\circ \cE - \cA(\psi_\xi\circ \cE) \Vert_{L^p(\mu)}
\ge 
c 
\Vert \psi_\xi - A(\psi_\xi) \Vert_{L^p([0,1]^d)},
\quad \forall \, \xi \in \Xi.
\]
Here $c>0$ is the constant appearing in \eqref{eq:c} (and is independent of $N$).
Taking expectations over $\xi$, it follows that 
\[
\E_\xi\left[
\Vert \Psi_\xi - \cA(\Psi_\xi) \Vert_{L^p(\mu)}
\right]
\ge 
c 
\E_\xi\left[
\Vert \psi_\xi - A(\psi_\xi) \Vert_{L^p([0,1]^d)}
\right].
\]
By construction of $\psi_\xi$, the right hand side is lower bounded by $c\kappa N^{-\lambda}$, with $\kappa > 0$ independent of $N$ and $\lambda = \frac{1}{p} + \frac{1}{d}\cdot \frac{\alpha}{\alpha+\L}$. Thus it follows that 
\[
\E_\xi\left[
\Vert \Psi_\xi - \cA(\Psi_\xi) \Vert_{L^p(\mu)}
\right]
\ge c \kappa N^{-\lambda},
\]
where $c,\kappa > 0$ are independent of $N$. Replacing $c \kappa $ by $\kappa$, the claimed bound \eqref{eq:lemLpop2} follows.
\end{proof}

The proof of Proposition \ref{prop:Lpop} is now immediate:

\begin{proof}[Proof of Proposition \ref{prop:Lpop}]
Proposition \ref{prop:Lpop} follows from Lemma \ref{lem:Lpop2} and Lemma \ref{lem:randest}.
\end{proof}

\subsection{Deep Operator Networks (DeepONet)}
\label{sec:don}

In this section, we state and prove a theory-to-practice gap for a general family of ``DeepONet'' architectures. We recall that we are interested in the approximation of operators $\cG: \cX \to \R$. In this setting, we define these DeepONet architectures to be of the form $\Psi = \psi\circ \cL$, combining a linear encoder $\cL: \cX \to \R^d$ with a feedforward neural network $\psi: \R^d \to \R$. The next three paragraphs provide a precise description of the considered architecture, define relevant approximation spaces, and prove an infinite-dimensional theory-to-practice gap for these architectures.

\paragraph{Architecture}

Fix a sequence of continuous linear functionals $\ell_1,\ell_2,\dots: \cX \to \R$. For $d_0\in \N$, we denote by $\cL_{d_0}$ the linear encoder $\cL_{d_0}: \cX \to \R^{d_0}$, $\cL_{d_0}(u) := (\ell_1 (u), \dots, \ell_{d_0}(u))$. The \define{encoder-net} $\Psi: \cX \to \R$ associated with a neural network $\psi: \R^{d_0}\to \R$ is a mapping of the form $\Psi(u) = \psi \circ \cL_{d_0}(u)$.

To ensure universality of the resulting operator learning architecture, we will make the (minimal) assumption that 
\begin{align}
\label{eq:span}
\mathrm{span}\set{\ell_j: \cX \to \R}{j\in \N} \subset \cX^\ast \text{ is dense},
\end{align}
where we recall that $\cX^\ast$ denotes the continuous dual of $\cX$. Throughout the following discussion, we will consider the sequence $(\ell_j)_{j\in \N}\subset \cX^\ast$ fixed, and it will be assumed that \eqref{eq:span} holds without further mention.

\paragraph{DeepONet Approximation Space}

We can now define spaces $\A^{\alpha,\infty}_{\bl,\DON}$ for DeepONets. To this end, we introduce,
\[
\Sigma^{\bl}_{n,\DON} 
:= 
\left\{\Psi = \psi \circ \cL_{d_0} : 
\begin{array}{l}
\psi \text{ NN with } d_{\text{in}}(\psi) = d_0, d_{\text{out}}(\psi) = 1, \\
\max\{W(\psi),d_0\}\leq n, L(\psi) \leq \bl(n), \|\psi\|_{\NN} \leq 1
\end{array}
\right\}.
\]
Then, given $\alpha \in (0,\infty),$ for each continuous (non-linear) operator $\cG: \cX \to \R$, we define
\[
\Gamma^{\alpha,\infty}_\DON(\cG) := \max \left\{ \sup_{u\in \cX} \|\cG(u)\|,\ \sup_{n \in \mathbb{N}} \left[ n^{\alpha} \cdot d_{\infty}\left(\cG,\Sigma_{n,\DON}^{\bl}\right) \right] \right\} \in [0,\infty],
\]
where $d_{\infty}(\cG,\Sigma) := \inf_{\Psi \in \Sigma} \sup_{u\in \cX} \|\cG(u) - \Psi(u)\|.$
We can define a DeepONet \define{approximation space quasi-norm} $\| \cdot \|_{\A_{\bl,\DON}^{\alpha,\infty}}$ by
\[
\|\cG\|_{\A_{\bl,\DON}^{\alpha,\infty}} := \inf \left\{\theta > 0 : \Gamma_\DON^{\alpha,\infty}(\cG/\theta) \leq 1\right\} \in [0,\infty],
\]
giving rise to the DeepONet \define{approximation space}
\[
\A_{\bl,\DON}^{\alpha,\infty} := \left\{\cG \in C(\cX) : \|\cG\|_{\A_{\bl,\DON}^{\alpha,\infty}} < \infty\right\}.
\]

\paragraph{Encoder Construction}

The following is a useful technical lemma, which will be applied to construct suitable encoders $\cE: \cX \to \R^d$. It shows that if a finite-dimensional map $F: V\subset \R^d\to \R^d$ is sufficiently close to the identity, then the image of $V$ must ``fill out'' a non-empty open set $V_0 \subset \R^d$.
\begin{lemma}
\label{lem:onto}
Let $V \subset \R^d$ be a non-empty domain. There exist constants $\epsilon_0, c_0 > 0$ and a non-empty open subset $V_0 \subset V$ with the following property: For any Lipschitz-continuous function $F: V \to \R^d$, satisfying
\[
\Vert F - \id \Vert_{W^{1,\infty}(V)} \le \epsilon_0,
\]
where $\id: V\to \R^d$, $\id(y) = y$ denotes the identity mapping,
it follows that $F(V) \supset V_0$, and 
\[
F_\#\unif(V)  \ge c_0 \unif(V_0).
\]
\end{lemma}

The result of Lemma \ref{lem:onto} follows as a consequence of the contraction mapping theorem; we provide a detailed proof in Appendix \ref{app:onto}. Our goal in this section is to construct encoders $\cE: \cX \to \R^d$ which ``fill out'' the set $[0,1]^d$. The link with the finite-dimensional setting of Lemma \ref{lem:onto} is made by identifying $\cX \simeq \R^d \times \Omega_d$, as in the decomposition of $\mu$ made in Assumptions \ref{ass:mud}, and by considering the second factor as a parameter. This leads us to study parametrized mappings, $F: V \times \Omega_d \to \R^d$, $(y,\xi) \mapsto F_\xi(y) = F(y;\xi)$, with the parameter $\xi\in \Omega_d$ a random variable. The following result derives a similar result as Lemma \ref{lem:onto} in this parametrized setting:

\begin{lemma}
\label{lem:onto-param}
Let $V\subset \R^d$ be a non-empty domain, and let $\epsilon_0, c_0>0$ and $V_0\subset V$ denote the constants and set of Lemma \ref{lem:onto}, respectively. Let $(\Omega, \P)$ be a probability space, and assume that $F: V \times \Omega \to \R^d$, $(y,\xi) \mapsto F(y;\xi)$ is measurable. Assume furthermore, that there exists $K\subset \Omega$, such that the mapping $F_\xi: V \to \R^d$, $y\mapsto F_\xi(y) := F(y;\xi)$ is Lipschitz for each $\xi\in K$, and
\begin{align}
\label{eq:onto-param}
 \Vert F_\xi - \id \Vert_{W^{1,\infty}(V)} \le \epsilon_0, \quad \forall \, \xi \in K.
\end{align}
Then the push-forward under $F$ of the product measure $\unif(V)\otimes \P$ on $V\times \Omega$ satisfies,
\[
F_\#(\unif(V)\otimes \P) \ge c_0 \P(K) \,  \Unif(V_0).
\]
\end{lemma}

A detailed proof of Lemma \ref{lem:onto-param} is given in Appendix \ref{app:onto}. Let now $\{\ell_k\}\subset \cX^\ast$ be a set of encoding functionals, such that $\Span\{\ell_k\} \subset \cX^\ast$ is dense. 
Our first goal is to use the $\{\ell_k\}$ to construct an encoder $\cE: \cX \to \R^d$, such that $\cE_{\#} \mu \ge c \unif([0,1]^d)$.

To this end, let us momentarily fix $\delta > 0$. Then by density, there exists $d_0\in \N$ and coefficients $c_{jk}$ for $j\in [d]$, $k\in [d_0]$, such that 
\begin{align}
\label{eq:enc-approx}
\left \Vert
e^\ast_j - \sum_{k=1}^{d_0} c_{jk} \ell_k
\right \Vert_{\cX^\ast} 
\le \delta, \quad \forall \, j=1,\dots, d.
\end{align}
Since $e_j^\ast(u(y;\xi)) = y_j$, \eqref{eq:enc-approx} allows us to approximate a ``projection'' onto $y_j$. Motivated by this, we now define the encoder $\cE: \cX \to \R^d$, $\cE(u) := (\cE_1(u),\dots, \cE_d(u))$, via
\begin{align}
\label{eq:enc-don}
\cE_j(u) = 
b_j +
\sum_{k=1}^{d_0} a_{jk} \ell_k(u),
\end{align}
for coefficients $a_{jk}$ and bias $b_j$, for $j\in [d]$ and $k\in [d_0]$, to be determined.

\begin{proposition}
\label{prop:encoder-don}
Assume that $\mu \in \cP(\cX)$ satisfies Assumption \ref{ass:mud} for $d\in \N$. 
Then there exists an encoder $\cE: \cX \to \R^d$ of the form \eqref{eq:enc-don} and constant $c>0$, such that 
\begin{align}
\label{eq:encoder-don}
\cE_\#\mu \ge c \, \Unif([0,1]^d).
\end{align}
\end{proposition}

The proof of Proposition \ref{prop:encoder-don} is a straight-forward, albeit somewhat tedious, consequence of Lemma \ref{lem:onto-param} and the fact that encoders of the form \eqref{eq:enc-don} are dense in the space of all affine encoders with $d$-dimensional range. The proof relies on the assumption that the linear functionals $\{\ell_k\}_{k\in \N}$ are dense in $\cX^\ast$. We include the detailed argument in Appendix \ref{app:encoder-don}.

\paragraph{Theory-to-Practice Gap}
We can now state a theory-to-practice gap for the unit ball $\U^{\alpha,\infty}_{\bl,\DON}$ in the DeepONet approximation space $\A^{\alpha,\infty}_{\bl,\DON}$:
\begin{theorem}[DeepONet theory-to-practice gap]
\label{thm:main-don}
Let $p\in [1,\infty]$. Let $\bl: \N \to \N\cup \{\infty\}$ be non-decreasing with $\bl^\ast \ge 4$. Assume that $\mu \in \cP(\cX)$ satisfies Assumption \ref{ass:mud}. Then for any $\alpha > 0$, we have
\begin{align}
\label{eq:main-don}
\beta_\ast(\U^{\alpha,\infty}_{\bl,\DON}, L^p(\mu)) \le \frac{1}{p}.
\end{align}
\end{theorem}

We recall that the typical Monte-Carlo (MC) approximation rate in the $L^p$-norm is $\beta_{MC} = 1/p$, reducing to the well-known $1/2$-rate with respect to the $L^2$-norm. Theorem \ref{thm:main-don} shows that, \emph{independently of $\alpha> 0$ and the depth $\bl^\ast \ge 4$}, it is not possible to achieve better-than-MC rates by any approximation method on the relevant approximation spaces $\A^{\alpha,\infty}_{\bl,\DON}$. 

\begin{proof}
Fix $d\in \N$, and let $\bl: \N \to \N\cup \{\infty\}$ be a non-decreasing function with $\bl^\ast \ge 4$. We denote $\tilde{\bl} := \bl -3$, so that $\tilde{\bl}^\ast \ge 3$, as in the assumptions of the \emph{finite dimensional} theory-to-practice gap, Theorem \ref{thm:main}. Let $\cE: \cX \to \R^d$ be the encoder of Proposition \ref{prop:encoder-don}, such that $\cE_\#\mu \ge c \, \Unif([0,1]^d)$. Let $\gamma_d \cdot \U^{\alpha,\infty}_{\bl,\DON}$ denote re-scaling of $\U^{\alpha,\infty}_{\bl,\DON}$ by a constant scaling factor $\gamma_d$. By Lemma \ref{lem:dilation}, which we state below after the proof, there exists a constant $\gamma_d \ge 1$, such that we have
\begin{align}
\label{eq:claim-don}
\set{f \circ \cE}{f \in U^{\alpha,\infty}_{\tilde{\bl}}(\R^d)} \subset \gamma_d \cdot \U^{\alpha,\infty}_{\bl,\DON}.
\end{align}
 Assuming \eqref{eq:claim-don}, the claim of Theorem \ref{thm:main-don} then follows immediately from Proposition \ref{prop:Lpop}. Indeed, defining $\U := \gamma_d \cdot \U^{\alpha,\infty}_{\bl,\DON}$, that proposition implies that 
\[
\beta_\ast(\gamma_d \cdot \U^{\alpha,\infty}_{\bl,\DON}, L^p(\mu)) \le \frac{1}{p} + \frac{1}{d} \cdot \frac{\alpha}{\alpha + \lfloor \tilde{\bl}^\ast/2\rfloor}.
\]
However, it follows from the definition that $\beta_\ast$ is invariant under re-scaling,
\[
\beta_\ast(\gamma_d \cdot \U^{\alpha,\infty}_{\bl,\DON}, L^p(\mu))
=
\beta_\ast(\U^{\alpha,\infty}_{\bl,\DON}, L^p(\mu)).
\]
Thus, recalling also $\tilde{\bl}^\ast = \bl^\ast -1$, we have 
\begin{align}
\label{eq:don-lpd}
\beta_\ast(\U^{\alpha,\infty}_{\bl,\DON}, L^p(\mu)) \le \frac{1}{p} + \frac{1}{d} \cdot \frac{\alpha}{\alpha+\lfloor ({\bl}^\ast -1)/2\rfloor}.
\end{align}
Since $d\in \N$ was arbitrary and the left-hand side is independent of $d$, we can take the infimum over all $d\in \N$ on the right to conclude that 
\[
\beta_\ast(\U^{\alpha,\infty}_{\bl,\DON}, L^p(\mu)) \le \frac{1}{p}.
\]
This is \eqref{eq:main-don}.


\end{proof}

\begin{lemma}\label{lem:dilation}
    Let $D\subset \R^d$, $\bl:\N\to \N\cup\{\infty\}$ non-decreasing and $f\in U^{\alpha,\infty}_{\bl}(D)$. Then for every $e\in \N$, $C\in \R^{d\times e}$ and $b\in \R^d$ there is $R\in (0,\infty)$ with $f(C\cdot + b) \in R\cdot U^{\alpha,\infty}_{\bl}(E)$, where $E:=\left\{x\in \R^{e}:\ Cx + b \in D\right\}\subset \R^{e}$.
\end{lemma}
The detailed proof of Lemma \ref{lem:dilation} is given in Appendix \ref{app:dilation}. We also state the following theory-to-practice gap for uniform approximation over compact $\cK$:

\begin{theorem}[DeepONet; uniform theory-to-practice gap]
\label{thm:unif-don}
Let $\bl: \N \to \N\cup\{\infty\}$ be non-decreasing with $\bl^\ast\ge 4$. Assume that $\cK\subset \cX$ is a compact set satisfying Assumption \ref{ass:K}. Then for any $\alpha > 0$, we have
\[
\beta_\ast(\U^{\alpha,\infty}_{\bl,\DON}, C(\cK)) = 0.
\]
\end{theorem}

\begin{proof}
By Proposition \ref{prop:reduction}, for any $d\in \N$, there exists a probability measure $\mu\in \cP(\cX)$, with $\supp(\mu) \subset \cK$, and $\mu$ satisfies Assumption \ref{ass:mud}. By Proposition \ref{prop:encoder-don}, there exists a DeepONet encoder $\cE:\cX \to \R^d$ and constant $c>0$, such that 
\[
\cE_\# \mu \ge c \, \unif([0,1]^d).
\]
Following the steps in the proof of Theorem \ref{thm:main-don}, leading up to \eqref{eq:don-lpd}, it follows that for any $p\in [1,\infty]$, we have 
\[
\beta_\ast(\U^{\alpha,\infty}_{\bl,\DON},L^p(\mu)) \le \frac{1}{p} + \frac{1}{d} \cdot \frac{\alpha}{\alpha + \lfloor (\bl^\ast - 1)/2\rfloor}.
\]
We now recall that,
\[
\beta_\ast(\U^{\alpha,\infty}_{\bl,\DON},C(\cK))
\le 
\beta_\ast(\U^{\alpha,\infty}_{\bl,\DON},L^p(\mu)).
\]
This inequality is \eqref{eq:reduction-ran} and follows from the fact that uniform approximation over $\cK$ is a more stringent criterion than $L^p(\mu)$ approximation with respect to $\mu$, owing to the fact that $\supp(\mu) \subset \cK$. Thus, we have 
\[
\beta_\ast(\U^{\alpha,\infty}_{\bl,\DON},C(\cK))
\le \frac{1}{p} + \frac{1}{d} \cdot \frac{\alpha}{\alpha + \lfloor (\bl^\ast - 1)/2\rfloor}.
\]
This holds for any $d\in\N$ and $p\in [1,\infty)$. The convergence rate $\beta_\ast(\U^{\alpha,\infty}_{\bl,\DON},C(\cK))$ on the left depends on $\alpha$ and $\bl$, but is independent of $p$ and $d$. Thus, upon letting $d,p\to \infty$, the claim follows.
\end{proof}

\subsection{Integral-kernel Neural Operators} 
\label{sec:no}

In this section, we state and prove a theory-to-practice gap for a general family of integral-kernel neural operator (NO) architectures \cite{kovachki2023neural,li2021fourier}. Again, we are interested in the approximation of operators $\cG: \cX \to \R$. In this setting, we define integral-kernel NO architectures to be of the form $\Psi = \cQ \circ \cL_L\circ \dots \cL_1 \circ \cR$, combining a lifting layer $\cR$, hidden layers $\cL_1,\dots, \cL_L$ and an output layer $\cQ$. The next three paragraphs provide a precise description of the considered architecture, define relevant approximation spaces, and prove an infinite-dimensional theory-to-practice gap.

\paragraph{Architecture}

The following is a minimal architecture shared by all/most variants of integral kernel-based neural operators \cite{LLS2024}. For notational simplicity, we focus on real-valued input and output functions. All results extend readily to the more general vector-valued case.

\begin{definition}[Averaging Neural Operator]
    Let $\cX(D;\R)$, $\cY(D;\R)$, and $\cV(D;\R^{\dc})$ be spaces of functions on Lipschitz domain $D \subset \R^{d_{D}}$. An averaging neural operator (ANO) $\Psi: \; \cX(D; \R) \to \cY(D;\R)$ of depth $L$ takes the form
    \begin{equation}\label{eqn:ANO_struct}
        \Psi(u) = \cQ \circ \cL_L\circ \dots \cL_1 \circ \cR(u)
    \end{equation}
    where $u \in \cX(D;\R)$ and $x \in \R^d$. In addition, the pointwise lifting operators $\cR$ and $\cQ$ are obtained by composition with shallow ReLU neural networks; i.e. there exist neural networks $R: \R \times \R^d \to \R^{d_c}$ and $Q: \R^{d_c} \to \R$, of depth $L(R) = L(Q) = 2$, such that 
    \begin{align}
    \cR(u)(x) = R(u(x),x), 
    \qquad
    \cQ(v)(x) = Q(v(x)).
    \end{align}
    Finally, the hidden layers $\cL_j: \cV(D; \R^{d_c}) \to \cV(D; \R^{d_c})$ take the form 
    \begin{equation}\label{eqn:ANO_layer}
        \cL_j(v)(x) = \sigma\left(W_j v(x) + b_j + \aint_{D} v(y) \dy \right),
    \end{equation}
    where $W_j \in \R^{d_c\times d_c}$ is a matrix and $b_j \in \R^{d_c}$ a bias.
\end{definition}

Generalizing our definitions of quantities of interest from Section \ref{sec:relu}, we will denote by $L(\Psi) := L$ the depth (number of hidden layers) of an ANO, and we denote by $W(\Psi) = W(R) + \sum_{j=1}^L \left(\Vert W_j \Vert_{\ell^0} + \Vert b_j \Vert_{\ell^0} \right) + W(Q)$ the total number of non-zero parameters of the architecture. Furthermore, we define $\Vert \Psi\Vert_{\NN} := \max\left\{\Vert W_j \Vert_{\infty}, \Vert b_j \Vert_\infty, \Vert R\Vert_{\NN}, \Vert Q \Vert_{\NN} \right\}$ as the maximal weight magnitude.

\begin{remark}
\label{rem:ano-in-fno}
The ANO introduced above is a special case of a more general family of kernel-based neural operators introduced in \cite{kovachki2023neural}. Its theoretical significance is that most instantiations of such neural operators contain the ANO as a special case, with a specific tuning of the weights. For example, the popular Fourier neural operator (FNO) \cite{li2021fourier} uses the same general structure, but employs hidden layers of the form 
\[
\cL(v)(x) = \sigma \left(
W v(x) + b + \int \kappa(x-y) v(y) \, dy
\right),
\]
where the integral kernel $\kappa$ is convolutional, and 
\[
\kappa(x) = \sum_{|k|\le k_{\mathrm{max}}} \hat{\kappa}_k e^{ik\cdot x}
\]
is parametrized by the coefficients $\hat{\kappa}_k \in \C^{d_c \times d_c}$ in its (truncated) Fourier expansion. Thus, the ANO can be obtained from the FNO upon setting $\hat{\kappa}_k \equiv 0$ for $k\ne 0$ and $\hat{\kappa}_0 = I_{d_c\times d_c} /\mathrm{vol}(D)$.
\end{remark}

\paragraph{NO Approximation Space}
We now define the relevant approximation spaces $\A^{\alpha,\infty}_{\bl,\NO}$ for (averaging) neural operators. Assume we are given function spaces $\cX(D) = \cX(D;\R)$ and $\cY(D) = \cY(D;\R)$. In our discussion, we will assume that $\cX(D) \subset L^\infty(D)$ is an infinite-dimensional Banach space on Lipschitz domain $D \subset \R^{d_{D}}$, and we will assume that $\cY(D)$ contains all constant functions. We now introduce,
\[
\Sigma^{\bl}_{n,\NO} 
:= 
\left\{\Psi: \cX \to \cY : 
\begin{array}{l}
\Psi \text{ is an ANO with }W(\Psi) \leq n, \\
 L(\Psi) \leq \bl(n), \|\Psi\|_{\NN} \leq 1
\end{array}
\right\}.
\]
\begin{remark}
Part of the definition of $\Sigma^{\bl}_{n,\NO}$ is that for any $\Psi \in \Sigma^{\bl}_{n,\NO}$, we must have $\Psi(\cX) \subset \cY$. Non-trivial $\Psi$ exist, since we can readily construct averaging neural operators $\Psi$ of the form \eqref{eqn:ANO_struct}, for which the output $\Psi(u)$ is a \emph{constant-valued function}, for any input $u\in \cX$. Since $\cY$ contains constant functions by assumption, this implies that such $\Psi$ defines a map $\Psi: \cX\to \cY$. In our proofs, we will only ever consider $\Psi$ of this form, thus our results hold even when $\cY = \R$. 
\end{remark}
Given $\alpha \in (0,\infty),$ for each continuous (non-linear) operator $\cG: \cX \to \cY$, we define
\[
\Gamma^{\alpha,\infty}_\NO(\cG) := \max \left\{ \sup_{u\in \cX} \|\cG(u)\|,\ \sup_{n \in \mathbb{N}} \left[ n^{\alpha} \cdot d_{\infty}\left(\cG,\Sigma_{n,\NO}^{\bl}\right) \right] \right\} \in [0,\infty],
\]
where $d_{\infty}(\cG,\Sigma) := \inf_{\Psi \in \Sigma} \sup_{u\in \cX} \|\cG(u) - \Psi(u)\|_\cY.$
We define a NO \define{approximation space quasi-norm} $\| \cdot \|_{\A_{\bl,\NO}^{\alpha,\infty}}$ by
\[
\|\cG\|_{\A_{\bl,\NO}^{\alpha,\infty}} := \inf \left\{\theta > 0 : \Gamma_\DON^{\alpha,\infty}(\cG/\theta) \leq 1\right\} \in [0,\infty],
\]
giving rise to the NO \define{approximation space}
\[
\A_{\bl,\NO}^{\alpha,\infty} := \left\{\cG \in C(\cX;\cY) : \|\cG\|_{\A_{\bl,\NO}^{\alpha,\infty}} < \infty\right\}.
\]
We again denote by $\U^{\alpha,\infty}_{\bl,\NO}$ the unit ball in $\A_{\bl,\NO}^{\alpha,\infty}$.

\begin{remark}
As pointed out in Remark \ref{rem:ano-in-fno}, the ANO can be obtained by a special setting of the weights in the FNO. As a consequence, it can be shown that $\A_{\bl,\NO}^{\alpha,\infty} \subset \A_{\bl,\FNO}^{\alpha,\infty}$, where $\A_{\bl,\FNO}^{\alpha,\infty}$ denotes the relevant approximation space for FNO, which can be defined in analogy to $\A_{\bl,\NO}^{\alpha,\infty}$. Based on this relationship, it could be shown that 
\[
\beta_\ast(\U_{\bl,\FNO}^{\alpha,\infty}, L^p(\mu)) 
\le
\beta_\ast(\U_{\bl,\NO}^{\alpha,\infty}, L^p(\mu)),
\]
and hence any upper bound on $\beta_\ast(\U_{\bl,\NO}^{\alpha,\infty}, L^p(\mu))$ implies a corresponding upper bound for the FNO.
\end{remark}
\paragraph{Encoder Construction} Let $\mu \in \cP(\cX)$ be a probability measure on $\cX$. We recall that $L^1(D)$ is a subset of the dual of $L^\infty(D)$ under the natural pairing,
\[
\langle u, e^\ast\rangle = \int_D u(x) e^\ast(x) \, \dx, \quad \forall u\in L^\infty(D), \; e^\ast \in L^1(D).
\]
The following proposition constructs an encoder $\cE: \cX \to \R^d$, whose existence will imply a theory-to-practice gap for the averaging neural operator.

\begin{proposition}\label{prop:encoder-no}
    Let $\mu$ satisfy Assumption \ref{ass:mud} for $d \in \N$, with bi-orthogonal elements $e^\ast_j \in L^1(D)$. Then there exists an encoder $\cE: \cX(D) \to \R^{d}$, 
    \begin{align}
    \label{eq:enc_exist}
    \cE(u) = \fint_{D} R(u(x),x) \, \dx,
    \end{align}
    with $R: \R \times \R^{d_D} \to \R^d$ a shallow ReLU neural network, 
    and constant $c >0$ dependent on $d$, such that
    \begin{align}
    \label{eq:encoder-no}
    \cE_{\#}\mu \geq c \cdot \Unif([0,1]^d).
    \end{align}
\end{proposition}

\begin{proof}
    It will suffice to show that there exists a neural network $R: \R\times \R^{d_D}\to \R^d$ and encoder of the form \eqref{eq:enc_exist},  such that for \emph{some} non-empty open set $V_0 \subset \R^d$ and constant $c>0$, we have
    \begin{align}
    \label{eq:encoder-no1}
    \cE_\#\mu \ge c \cdot \Unif(V_0).
    \end{align}
    Indeed, given such $V_0$, there exists a scaling factor $\gamma > 0$ and shift $b \in \R^d$, such that $[0,1]^d \subset \gamma \cdot V_0 + b$. Replacing the neural network $R(\eta,x)$ by $\tilde{R}(\eta,x) := \gamma \cdot R(\eta,x) + b$, it is then immediate that the encoder $\tilde{\cE}: \cX(D) \to \R^d$ defined by $\tilde{\cE}(u) = \fint_{D} \tilde{R}(u(x),x) \dx$ satisfies a lower bound of the form \eqref{eq:encoder-no}.

    To prove the existence of an encoder $\cE$ satisfying \eqref{eq:encoder-no1}, we recall that, by Assumption \ref{ass:mud}, $u \sim \cP(\cX)$ is of the form 
    \begin{align*}
    u(x;y,\xi) = \xi(x) + \sum_{j=1}^{d} y_j e_j(x),
    \end{align*}
    where $e_1,\dots, e_d$ are linearly independent with bi-orthogonal elements $e^\ast_1,\dots, e^\ast_d$, and the coefficients $y_j \sim \rho_j(y) \dy$ are independent. Furthermore, $\xi \sim \mu_d^\perp$ is a random function such that $e^\ast_1(\xi) = \dots = e^\ast_d(\xi) = 0$. In the following we will consider $\xi$ a random ``parameter'' and denote the law of $\xi$ by $\P := \mu_d^\perp$. By assumption the dual elements $e^\ast_j$ are represented by a function in $L^1$. 
    
    Under our assumptions, there exist non-empty intervals $I_j \subset \R$ and constant $c_\rho >0$, such that $\essinf_{I_j} \rho_j(z) \ge c_\rho$ for all $j=1,\dots, d$.  
    We may assume without loss of generality that $I_j$ is a bounded interval, and fix a constant $c_V > 0$ such that $I_j \subset [-c_V,c_V]$ for all $j=1,\dots, d$. Let $V = \prod_{j=1}^d I_j \subset \R^d$. We now choose $B>0$, such that $\Prob(\Vert \xi \Vert_{L^\infty} \le B) > 0$. This is possible, because of the assumed inclusion $\cX \subset L^\infty$. For this choice of $B>0$, we define the random event
    \[
    \cK := \{ \Vert \xi \Vert_{L^\infty} \le B \},
    \]
    so that $\P(\cK) > 0$. Note that this bound on $B$ implies that for all $y\in V$ and $\xi \in K$, we have 
    \[
    \Vert u(\slot;y,\xi) \Vert_{L^\infty(D)} \le B + d c_V \max_{j=1,\dots, d} \Vert e_j \Vert_{L^\infty(D)} =: B'.
    \]

For fixed $\xi \in K$, define the maps $F_\xi^\dagger, F_\xi: \R^d \to \R^d$, as follows:
\begin{align}
\label{eq:fxi1}
\left\{\; 
\begin{aligned}
    F_{\xi}^\dagger(y) &:= \left\{\aint_D u(x;y,\xi) \, e^*_k(x) \dx\right\}_{k=1}^d = y\\
    F_{\xi}(y) &:= \aint_D \; R\left(u(x;y,\xi),x\right)\dx
    \end{aligned}
    \right.
\end{align}
where $R$ is a ReLU neural network mapping $\R \times \R^{d_{D}}$ to $\R^d$. Let $R^\dagger$ be defined by $R^\dagger(\eta,x) = \{ \eta \, e^*_k(x)\}_{k=1}^d$. Then by Corollary \ref{cor:LinfW1inf}, for any $\epsilon > 0$, there exists a ReLU neural network $R$ such that
\begin{align}\label{eqn:ReLU-result}
\aint_{D} \|R(\cdot, x)- R^\dagger(\cdot,x)\|_{W^{1,\infty}([-B',B'];\R^d)} \dx \leq \epsilon.
\end{align}
Identify $R$ with a ReLU neural network achieving this bound. 
Note that $F_{\xi}^\dagger \equiv \id$ is exactly the identity on $\R^d$ for all $\xi \in K$. Given the constant $\epsilon_0 > 0$ of Lemma \ref{lem:onto-param}, We seek to show that for sufficiently small $\epsilon >0$ in \eqref{eqn:ReLU-result}, we can ensure that
\begin{align}
\label{eqn:ReLU-result-implies}
\|F_\xi-F_\xi^\dagger\|_{W^{1,\infty}(V)} = \|F_\xi-\id\|_{W^{1,\infty}(V)} \le \epsilon_0, \quad \forall \, \xi \in K.
\end{align}
By Lemma \ref{lem:onto-param}, this entails that there exist $V_0\subset V$ and $c_0>0$, such that
\[
F_\#\left(\Unif(Y) \otimes \P \right)
\ge c_0 \P(K) \, \Unif(V_0),
\]
where $\P$ denotes the law of $\xi$. The claim then follows by observing that 
\[
\cE_\# \mu 
= F_\#\left( \mu_d\otimes \P \right)
\ge c_\rho^d \, F_\# \left( \Unif(Y) \otimes \P\right)
\ge c_\rho^d c_0 \P(K) \, \Unif(V_0).
\]
Since $\P(K) > 0$ by construction, the claim then follows with constant $c := c_\rho^d c_0 \P(K)$. It therefore remains to show that \eqref{eqn:ReLU-result} for sufficiently small $\epsilon > 0$ implies \eqref{eqn:ReLU-result-implies}. 

In the remainder of this proof, we will show that this holds for $\epsilon := \epsilon_0 / 2d$. By \eqref{eqn:ReLU-result}, we have $\|\p_{\eta}(R-R^\dagger)(\cdot,x)\|_{L^{\infty}([-B',B'])} < \epsilon$, where $\eta$ refers to the first argument of $R$ and $R^\dagger$. Then for $\xi \in K$ and $y,y' \in V$, we have 
\[
\Vert u(\slot;y,\xi) \Vert_{L^\infty(D)},
\; 
\Vert u(\slot;y,\xi) \Vert_{L^\infty(D)} \le B',
\]
and hence
\begin{align*}
\Vert F_\xi - F_\xi^\dagger \Vert_{L^\infty(V)}
&\le 
\int_D \Vert R(\slot,x) - R^\dagger(\slot,x) \Vert_{L^\infty([-B',B'])}\dx
\le \epsilon = \epsilon_0 / 2d.
\end{align*}
To estimate $\Vert DF_\xi - DF_\xi^\dagger \Vert_{L^\infty(V)}$, we recall that, due to the convexity of the $d$-dimensional cube $V$, the $W^{1, \infty}(V)$ seminorm is equal to the Lipschitz seminorm: 
\begin{equation*}
    \|D F_\xi\|_{L^{\infty}(V)} = \sup_{y, y' \in V} \frac{|F_\xi(y) - F_\xi(y')|}{|y-y'|}.
\end{equation*}
We now bound, for $y,y'\in V$:
\begin{align*}
    |(F_\xi-F_\xi^\dagger)(y) - (F_\xi - F_\xi^\dagger)(y')| &= \left|\aint_D (R-R^\dagger)(u(x;y,\xi), x) - (R-R^\dagger)(u(x;y',\xi),x) \dx\right|\\
    & \leq \aint_D \left|(R-R^\dagger)(u(x;y,\xi), x) - (R-R^\dagger)(u(x;y',\xi),x)\right| \dx \\
    & \leq \aint_D \left|u(x;y,\xi)-u(x;y',\xi)\right|\|\p_{\eta}(R-R^\dagger)(\cdot,x)\|_{L^{\infty}([-B',B'])} \dx \\
    & = \aint_D \left|\sum_{j=1}^d (y_j - y_j') e_j(x)\right|\|\p_\eta (R-R^\dagger)(\cdot,x)\|_{L^{\infty}([-B',B'])} \dx\\
    & \leq \aint_D \sum_{j=1}^d |y_j - y_j'| \|\p_{\eta}(R-R^\dagger)(\cdot,x)\|_{L^{\infty}([-B',B'])}\dx \\
    & \leq \sqrt{d}|y - y'| \epsilon.
\end{align*}
With our choice of $\epsilon = \epsilon_0/2d$, this implies that
\begin{align*}
    \|DF_\xi - DF^\dagger_{\xi}\|_{L^\infty(V)} & = \sup_{y, y' \in V} \frac{|(F_\xi-F_\xi^\dagger)(y) - (F_\xi - F_\xi^\dagger)(y')| }{|y - y'|} 
    \leq \epsilon\sqrt{d}
    \le \epsilon_0/2.
\end{align*}
Combining both estimates, we have shown that 
\[
\fint_D
\Vert R(\slot,x) - R^\dagger(\slot;x)\Vert_{W^{1,\infty}([-B',B'])} \dx
\le 
\epsilon,
\]
for $\epsilon = \epsilon_0 / 2d$, implies
\[
\Vert F_\xi - F_\xi^\dagger \Vert_{W^{1,\infty}(V)} \le \epsilon_0.
\]
This is what we set out to show, and concludes our proof of Proposition \ref{prop:encoder-no}.
\end{proof}

\paragraph{Theory-to-Practice Gap}We can now state a theory-to-practice gap for the unit ball $\U^{\alpha,\infty}_{\bl,\NO}$ in the NO approximation space $\A^{\alpha,\infty}_{\bl,\NO}$:
\begin{theorem}[NO theory-to-practice gap]
\label{thm:main-no}
Let $p\in [1,\infty]$. Let $\bl: \N \to \N\cup \{\infty\}$ be non-decreasing with $\bl^\ast \ge 4$. Let $\cX(D) \subset L^\infty(D)$ be a Banach space on Lipschitz domain $D \subset \R^{d_{D}}$. Assume that $\mu \in \cP(\cX)$ satisfies Assumption \ref{ass:mud} with bi-orthogonal elements $\{e^\ast_j\}_{j\in \N} \subset L^1(D)$. Then for any $\alpha > 0$, we have
\begin{align}
\label{eq:main-no}
\beta_\ast(\U^{\alpha,\infty}_{\bl,\NO}, L^p(\mu)) \le \frac{1}{p}.
\end{align}
\end{theorem}

The proof of Theorem \ref{thm:main-no} relies on the following lemma:
\begin{lemma}
\label{lem:inclusion-no}
Let $\cX(D) \subset L^\infty(D)$ be a Banach space on Lipschitz domain $D \subset \R^{d_{D}}$, and let $\mu \in \cP(\cX)$ be a probability measure on $\cX$. Assume that $\cY(D)$ contains all constant functions and $\mu$ satisfies Assumption \ref{ass:mud} for $d\in \N$, with bi-orthogonal elements $e^\ast_1,\dots,e^\ast_d\in L^1(D)$. Let $\cE: \cX(D) \to \R^d$ be the encoder of Proposition \ref{prop:encoder-no}. Fix a (finite) constant $\bl_0 \le \bl^\ast+2$. There exists a constant $\gamma = \gamma(d,\cE,\alpha,\bl_0) > 0$, such that 
\begin{align}
\label{eq:claim-no}
\set{f \circ \cE}{f \in U^{\alpha,\infty}_{\bl_0}(\R^d)} \subset \gamma \cdot \U^{\alpha,\infty}_{\bl,\NO}.
\end{align}
\end{lemma}
A proof of this lemma is given in Appendix \ref{app:inclusion-no}. We now come to the proof of Theorem \ref{thm:main-no}.

\begin{proof}[Proof of Theorem \ref{thm:main-no}]
Fix $d\in \N$, and let $\bl: \N \to \N\cup \{\infty\}$ be a non-decreasing function. Clearly, we have $\bl^\ast \ge 1$. We define $\bl_0 := 3$, so that the \emph{finite dimensional} theory-to-practice gap, Theorem \ref{thm:main} applies to $U^{\alpha,\infty}_{\bl_0}(\R^d)$. Let $\cE: \cX \to \R^d$ be an encoder as in Proposition \ref{prop:encoder-no}, with $\cE_\#\mu \ge c \, \Unif([0,1]^d)$. By \eqref{eq:claim-no}, there exists a constant $\gamma>0$, such that 
\[
\set{f \circ \cE}{f \in U^{\alpha,\infty}_{\bl_0}(\R^d)} \subset \gamma \cdot \U^{\alpha,\infty}_{\bl,\NO}.
\]
The claim of Theorem \ref{thm:main-no} then follows again from Proposition \ref{prop:Lpop}, as in the proof of Theorem \ref{thm:main-don}. Indeed, Proposition \ref{prop:Lpop}, and the fact that $\beta_\ast(\gamma \cdot \U^{\alpha,\infty}_{\bl,\NO}, L^p(\mu)) = \beta_\ast(\U^{\alpha,\infty}_{\bl,\NO}, L^p(\mu))$, imply that 
\[
\beta_\ast(\U^{\alpha,\infty}_{\bl,\NO}, L^p(\mu))
\le 
\frac{1}{p} + \frac{1}{d} \cdot \frac{\alpha}{\alpha + \lfloor \bl_0/2\rfloor}.
\]
Since the left-hand side is independent of $d$, we let $d\to \infty$, to obtain \eqref{eq:main-no}.
\end{proof}

We also state the following theory-to-practice gap for uniform approximation over compact $\cK$:

\begin{theorem}[NO; uniform theory-to-practice gap]
\label{thm:unif-no}
Let $\bl: \N \to \N\cup\{\infty\}$ be non-decreasing with $\bl^\ast\ge 4$. Let $\cX(D) \subset L^\infty(D)$ be a Banach space on Lipschitz domain $D \subset \R^{d_{D}}$. Assume that $\cK\subset \cX$ is a compact set satisfying Assumption \ref{ass:K}. Then for any $\alpha > 0$, we have
\[
\beta_\ast(\U^{\alpha,\infty}_{\bl,\NO}, C(\cK)) = 0.
\]
\end{theorem}

\begin{proof}
The proof is analogous to the argument for the uniform theory-to-practice gap for DeepONet, except that the DeepONet encoder construction, Proposition \ref{prop:encoder-don}, is replaced by the NO encoder construction in \ref{prop:encoder-no}, and the relevant inclusion is the one identified in Lemma \ref{lem:inclusion-no}.
\end{proof}

\section{Discussion and Conclusion}
\label{sec:conclusion}

This work has rigorously examined the theory-to-practice gap in both finite-dimensional and infinite-dimensional settings, resulting in rigorous bounds on achievable convergence rates for general reconstruction methods based on point-values. By deriving upper bounds on the optimal rate $\beta_\ast$, we have uncovered the inherent constraints of learning on relevant neural network and neural operator approximation spaces. In the finite-dimensional case, our contributions include a unified treatment of the theory-to-practice gap for approximation errors measured in general $L^p$-spaces for arbitrary $p \in [1, \infty]$ and dimension $d\in \N$. Furthermore, we extend the theory-to-practice gap to infinite-dimensional operator learning frameworks, and derive results for prominent architectures such as Deep Operator Networks and integral kernel-based neural operators, such as the Fourier neural operator (FNO). Notably, for operator learning we establish that the optimal convergence rate in a Bochner $L^p$-norm satisfies $\beta_\ast \leq 1/p$, while no algebraic convergence is possible ($\beta_\ast = 0$) for uniform approximation on infinite-dimensional compact input sets. These findings highlight some intrinsic limitations of these data-driven methodologies and provide a clearer understanding of the theoretical bounds shaping practical applications. \mt{In particular, the practical barrier to accurate operator learning models is an information-theoretic barrier rather than a model expressivity barrier. Ultimately, the parameters of any model must be identified from finite data samples, and the present work shows that in the worst case, the sample convergence rate will not match the parametric one. We hypothesize that the sample complexity limitation identified in this work provides an explanation for why sample convergence rates can be seen in practice; e.g. \cite{PtO}, while to our knowledge it is rare to match empirical parametric rates to rates obtained via theory in operator learning.}

\mt{ A natural response to the hardness results obtained in this work is what modifications to the problem setting would allow this limitation to be bypassed. In particular, the results of the present work state that in the worst case, no algebraic rate is possible for uniform approximation in operator learning, even in the setting of arbitrarily good parametric rates. As a closing discussion, we attempt to place this result in context of some other relevant results in the literature. In the nonparametric statistics literature, it has been proved that no algebraic rates exist for finite-data regression of Lipschitz functionals on a Hilbert space \cite{mas2012lower}. A similar result holds for approximation of Lipschitz operators between Banach spaces; no method based on finite data samples can achieve algebraic sample rates \cite{adcock2024learninglipschitzoperatorsrespect, NLM2024data}. A complementary result to the present work restricts the notion of the parametric space to efficient parametric approximation by a particular architecture (FNO), and finds that for error measured in $L^2$, algebraic parametric convergence rates imply algebraic sample convergence rates \cite{NLM2024data}. On the other hand, for holomorphic functions, existing work has provided both lower and upper bounds on the algebraic approximation rates that hold in this setting \cite{adcock2024optimal}. While it is in general difficult to directly compare rates results as the particular settings tend to differ, recent work has made an attempt to present existing sample rates for operator learning in a general setting \cite{brugiapaglia2026short}.}

\mt{As an alternative to modifying smoothness assumptions on the function space, another direction could be to modify the character of the data. While our results hold for any possible algorithm, deterministic or randomized or adaptive, using finite data samples, we restricted to the setting of point evaluations $f(x_j)$. An examination of the proofs of Lemma \ref{lem:void} and Lemma \ref{lem:randest} reveals that the situation does not improve if one instead has access to point samples of derivatives because the argument relies on detection of a localized bump function, and higher order derivatives of the localized bump function are still indistinguishable from those of the zero function outside the support of the localized bump function. However, it is an open question whether our lower bounds still hold if the data samples provide some \textit{nonlocal} information about the function. While we expect that they do in general, it is possible that simultaneously restricting the function class to be approximated and elevating the data samples to include nonlocal functionals could yield improved rates. Additionally, while the present work focused on the noiseless setting, if the approximation classes remain noiseless, we expect the lower bounds on sample complexity rates to hold in the setting of noisy data samples. However, it is unclear whether the same would hold true if the approximation classes that determine the parametric complexity were defined with some notion of noise from the onset.
}

There are several additional interesting avenues for future work, two of which we briefly mention in closing.
One open problem is to study the theory-to-practice gap under additional constraints, e.g. on spaces of the form $\A^{\alpha,\infty}_{\bl} \cap \Lip$. We expect that the theory-to-practice gap will persist essentially unchanged even when introducing additional regularity constraints. Another open problem is to extend this gap beyond ReLU activations. This is specifically relevant for operator learning, where popular implementations of e.g. FNO usually use a smooth variant of ReLU, such as GeLU. To date, even in finite dimensions, no theory-to-practice gap is known for such smooth activation functions, \mt{though such results do hold for the ReQu activation \cite{abdeljawad2023}.} Since our proofs rely on the homogeneity of ReLU, the path to such an extension for smooth activations is not immediately obvious.

\appendix 

\section{A result on neural network approximation}

 We here derive a technical result, which shows that ReLU neural networks can approximate $C^1$ functions in the $W^{1,\infty}$-norm over compact subsets. This result follows from well-known techniques, but we couldn't find a reference. We hence include a statement and proof (sketch), below.

\begin{lemma}\label{lem:C1-ReLU}
Let $\sigma(x) = \max(x,0)$ denote the ReLU activation. Let $D \subset \R^d$ be a bounded Lipschitz domain. For any $f \in C^1(D)$ and $\epsilon > 0$, there exists a shallow ReLU-neural network $\psi: \R^d\to \R$, of the form, 
\begin{equation}\label{eqn:ReLU-form}
\psi(x) = \sum_{j=1}^N a_j \sigma(w_j^T x + b_j), \quad \text{with } a_j, b_j \in \R, \; w_j \in \R^d, \text{ for } j=1,\dots, N,
\end{equation}
such that $\Vert \psi - f \Vert_{W^{1,\infty}(D)} \le \epsilon$.
\end{lemma}

\begin{proof}
\textbf{Step 1:} Fix a compactly supported smooth function $\rho: \R \to \R$, with $\supp(\rho) \subset (-1,1)$, $\rho \ge 0$, and such that $\rho$ is even, i.e. $\rho(x) = \rho(-x)$ for all $x\in \R$. Let $\sigma_\rho(x) := (\sigma \ast \rho)(x)$ denote the convolution. We note that $\sigma_\rho \in C^\infty(\R)$ is smooth and monotonically increasing and that $\sigma_\rho$ is equal to ReLU on $\R \setminus (-1,1)$, i.e.
\[
\sigma_\rho(x) = 
\begin{cases}
0, &(x\leq-1), \\
x, &(x\geq1).
\end{cases}
\]

We first note that for any $\epsilon > 0$, we can find $N\in \N$ and coefficients $\alpha_j, \beta_j, \omega_j \in \R$, for $j=1,\dots, N$, such that 
\[
\left\Vert
\sigma_\rho(x) 
-
\textstyle\sum_{j=1}^N \alpha_j \sigma(\omega_jx + \beta_j)
\right\Vert_{W^{1,\infty}(\R)} \le \epsilon.
\]
To see this, we temporarily fix $M\in \N$, introduce an equidistant partition $x_m := -1 + 2m/M$ of $[-1,1]$, for $m=0,\dots, M$, denote $c_m := \sigma_\rho'(x_m)-\sigma_\rho'(x_{m-1})$ and note that 
\begin{align}
\label{eq:cases}
\left|
\sigma_\rho'(x) - 
\sum_{m=1}^M c_m 1_{[x_{m},\infty)}(x)
\right|
=
\begin{cases}
|\sigma_\rho'(x) |, &x \in (-\infty,-1), \\
|\sigma_\rho'(x) - 
\sigma_\rho'(x_{m_0}) |, &x \in [x_{m_0},x_{m_0+1}), \\
|\sigma_\rho'(x) - \sigma_\rho'(x_M)|, &x \in [1,\infty).
\end{cases}
\end{align}
Since $\sigma_\rho'(x) \equiv 0$ for $x\leq-1$ and $\sigma_\rho'(x) \equiv 1$ for $x \ge x_M = 1$, it follows that 
\[
\left|
\sigma_\rho'(x) - 
\sum_{m=1}^M c_m 1_{[x_{m},\infty)}(x)
\right|
= 0, \quad \forall \, x \in \R \setminus [-1,1).
\]
On the other hand, we also have
\[
\Vert \sigma_\rho'' \Vert_{L^\infty(\R)} = \Vert \sigma' \ast \rho' \Vert_{L^\infty(\R)} \le \Vert \sigma'\Vert_{L^\infty(\R)} \Vert \rho' \Vert_{L^1(\R)} \le \Vert \rho' \Vert_{L^1(\R)}.
\]
For any $x\in [-1,1)$, we can find $m_0 \in \{0,\dots, M-1\}$, such that $x \in [x_{m_0},x_{m_0+1})$, and from \eqref{eq:cases}, we obtain
\begin{align*}
\left|
\sigma_\rho'(x) - 
\sum_{m=1}^M c_m 1_{[x_{m},\infty)}(x)
\right|
&\le
|\sigma'_\rho(x) - \sigma'_\rho(x_{m_0})|
\\
&\le 
\Vert \sigma'' \Vert_{L^\infty} |x-x_{m_0}|
\\
&\le \frac{2\Vert \rho' \Vert_{L^1(\R)}}{M}.
\end{align*}
Given $\epsilon > 0$, choose $M$ sufficiently large so that $2\Vert \rho' \Vert_{L^1(\R)}/ M \le \epsilon / 2$. Then, noting that $\sigma'(x-x_m) = 1_{[x_m,\infty)}(x)$ pointwise a.e., it follows that 
\[
\left\Vert 
\sigma_\rho'(x) - \sum_{m=1}^M c_m \sigma'(x-x_m) 
\right\Vert_{L^\infty(\R)} \le \epsilon / 2.
\]
Taking into account that $\sigma_\rho(x) \equiv 0 \equiv \sum_{m=1}^M c_m \sigma(x-x_m)$ for $x< -1$ and $\sigma_\rho(x) \equiv x \equiv \sum_{m=1}^M c_m \sigma(x-x_m)$ for $x> 1$,  this in turn implies that, 
\[
\left\Vert 
\sigma_\rho(x) - \sum_{m=1}^M c_m \sigma(x-x_m) 
\right\Vert_{L^\infty(\R)} 
\le 
\int_{-1}^1 \left\Vert 
\sigma_\rho'(x) - \sum_{m=1}^M c_m \sigma'(x-x_m) 
\right\Vert_{L^\infty(\R)} \, dx'
\le
\epsilon.
\]
Since $\epsilon$ was arbitrary, we have shown that $\sigma_\rho$ belongs to the $W^{1,\infty}(\R)$-closure of the set of shallow $\sigma$-neural networks, in one spatial dimension. In turn, this implies that any shallow $\sigma_\rho$-neural network in $d$ dimensions can be approximated by a shallow $\sigma$-neural network to any desired accuracy in the $W^{1,\infty}(\R^d)$-norm.

\textbf{Step 2:} Given a compact domain $D\subset \R^d$, it follows from \cite[Theorem 4.1]{Pinkus1999} and the fact that $\sigma_\rho$ is smooth and non-polynomial, that the set of shallow $\sigma_\rho$-neural networks is dense in $C^1(D)$. Thus, for any $f\in C^1(D)$ and given $\epsilon > 0$, we can first find a shallow $\sigma_\rho$-neural network $\psi_\rho$, such that 
\[
\Vert f - \psi_\rho \Vert_{W^{1,\infty}(D)} 
=
\Vert f - \psi_\rho \Vert_{C^{1}(D)} 
\le \epsilon / 2.
\]
Second, as a result of Step 1 we can find a $\sigma$-neural network $\psi$, such that 
\[
\Vert \psi_\rho - \psi \Vert_{W^{1,\infty}(D)} \le \epsilon /2.
\]
By the triangle inequality, we conclude that, for this $\psi$, we have
\[
\Vert f - \psi \Vert_{W^{1,\infty}(D)}
\le 
\Vert f - \psi_\rho \Vert_{W^{1,\infty}(D)} 
+
\Vert \psi_\rho - \psi \Vert_{W^{1,\infty}(D)}
\le \epsilon.
\]
\end{proof}
In the following corollary, we weaken the $C^1$ requirements for Lemma \ref{lem:C1-ReLU} for chosen inputs.

\begin{corollary} \label{cor:LinfW1inf} Let $D_1 \subset \R^{d_1}$ and $D_2\subset \R^{d_2}$ be compact domains. Let $f: D_1 \times D_2 \to \R$, $(\eta,x)\mapsto f(\eta,x)$ be a measurable function such that $f$ is $C^1$ in $\eta$ and integrable in $x$, such that
\[
\int_{D_2} \Vert f(\slot,x) \Vert_{C^{1}(D_1)} \dx < \infty.
\]
The for any $\epsilon >0$, there exists a shallow ReLU-neural network $\psi: \R^{d_1 + d_2} \to \R$ of the form \ref{eqn:ReLU-form} such that 
\begin{equation}\label{eqn:LinfW1inf-result}
    \int_{D_2} \|\psi(\cdot,x) - f(\cdot,x)\|_{W^{1,\infty}(D_1)} \dx \leq \epsilon. 
\end{equation}
\end{corollary}

\begin{proof}
The assumption says that $f$ belongs to $L^1_x(D_2;C^1_\eta(D_1))$. Clearly, we have $C^1(D_1\times D_2) \subset L^1_x(D_2;C^1_\eta(D_1))$. Upon mollifying $f(\eta,x)$ in the second variable, one checks that $C^1(D_1\times D_2)$ is in fact dense in $L^1_x(D_2;C^1_\eta(D_1))$. Thus, there exists $f_\epsilon \in C^1(D_1\times D_2)$, such that 
\[
\int_{D_2} \Vert f(\slot,x) - f_\epsilon(\slot,x) \Vert_{C^1(D_1)} \dx \le \epsilon/2.
\]
From Lemma \ref{lem:C1-ReLU}, there exists a shallow ReLU-neural network $\psi: \R^{d_1+d_2}\to\R$ such that 
\begin{equation*}
    \|\psi - f_{\epsilon}\|_{W^{1,\infty}(D_1 \times D_2)} \leq \epsilon/(2|D_2|).
\end{equation*}
This in turn implies that 
\begin{equation*}
    \int_{D_2} \|\psi(\cdot, x) - f_{\epsilon}(\cdot, x)\|_{W^{1,\infty}(D_1)} \dx \leq \epsilon/2. 
\end{equation*}
   Combining the above estimates, and using the triangle inequality, we conclude that
\begin{equation*}
    \int_{D_2} \Vert \psi(\slot,x) - f(\slot, x)\Vert_{W^{1,\infty}(D_1)} \dx \le \epsilon/2 + \epsilon/2 = \epsilon.
\end{equation*}
\end{proof}

\section{Proofs for Section \ref{sec:don}}

\subsection{Proof of Lemma \ref{lem:onto}}
\label{app:onto}

Our proof of Lemma \ref{lem:onto} will make use of the following version of the contraction mapping theorem:
\begin{lemma}[Lemma 6.6.6. of \cite{tao2006analysis}]\label{lem:tao}
    Let $B(0,r)$ be a ball in $\R^n$ centered at the origin and let $g: B(0,r) \to \R^n$ be a map such that $g(0) = 0$ and \[|g(x) - g(y)| \leq \frac{1}{2}|x-y| \text{ for all } x,y \in B(0,r).\] Then the function $F: B(0,r) \to \R^n$ defined by $F(x) = x + g(x)$ is one-to-one, and the image $F(B(0,r))$ of this map contains the ball $B(0,\frac{r}{2})$. 
\end{lemma}

We now come to the proof of Lemma \ref{lem:onto} in the main text.

\begin{proof}[Proof of Lemma \ref{lem:onto}]
Since $V\subset \R^d$ is open, we may pick $r > 0$ and $y_0\in V$ such that $B(y_0,r) \subset V$. We will show that the claim holds with 
\[
\epsilon_0 := \min\left(\tfrac12, \tfrac{r}{4}\right),
\quad
V_0 := B(y_0,\tfrac{r}{4}).
\]
Our proof relies on the contraction mapping theorem, formulated as Lemma \ref{lem:tao}. To this end, we first define $\tilde{F}: B(0,r) \to \R^d$, by
\[
\tilde{F}(y) = F(y+y_0) - F(y_0).
\]
We have that $\tilde{F}(0) = 0$ and, by assumption, the spectral norm of the Jacobian satisfies $\|D\tilde{F}(y) - I_d \|_{2} \le \Vert F - \id \Vert_{W^{1,\infty}(V)} \le \epsilon_0$ for all $y\in B(0,r)$. Defining $g = \tilde{F} - \id$, this implies that $\Vert Dg(y) \Vert_2 \le \epsilon_0$ for all $y\in B(0,r)$. We assume that $0 < \epsilon_0 \le 1/2$. Since,
\begin{align*}
    g(y) - g(y') & = \int_0^1 \frac{d}{dt} g(y' + t(y-y')) \dt \\
    & = \int_0^1 Dg(y' + t(y-y')) \dt \cdot (y - y'),
\end{align*}
this implies that, 
\begin{align*}
    |g(y) - g(y')| & \leq \int_0^1 \left\|Dg(y'+t(y-y'))\right\|_2 \dt |y - y'|\\
    & \leq \frac{1}{2} |y - y'|.
\end{align*}
for $y, y' \in B(0,r)$ by our assumption on $\epsilon_0$. As a consequence of the contraction mapping theorem (cp. Lemma \ref{lem:tao}), $\tilde{f}$ is injective on $B(0,r)$, and $B(0,\frac{r}{2}) \subset \tilde{F}(B(0,r))$.  
The next step is to return to $B(y_0, r)$. As 
\[
F(y + y_0) = \tilde{F}(y) + F(y_0),
\] 
and since $F(y_0)$ is a constant shift, $F$ is injective on $B(y_0, r)$, and $B(F(y_0), \frac{r}{2}) \subset F(B(y_0,r))$. The center of ball $B(F(y_0), \frac{r}{2})$ clearly depends on the value of $F(y_0)$. However, we argue that $B(y_0,\frac{r}{4}) \subset B(F(y_0),\frac{r}{2})$: this follows from the fact that, by assumption on $F$, we have
\begin{align*}
    |F(y_0) - y_0| &\le \Vert F - \id \Vert_{W^{1,\infty}(V)} \le \epsilon_0 \le \frac{r}{4}.
\end{align*}
Thus, $B(F(y_0), \frac{r}{2})$ contains a ball $B(y_0,r_0)$ of radius $r_0 = \frac{r}{2} - \epsilon_0 \ge \frac{r}{4}$. It follows that
\[
B(y_0,\tfrac{r}{4})
\subset 
B(F(y_0), \tfrac{r}{2})
\subset
F(B(y_0,r)).
\]
This shows that the image of $F: V \to \R^d$ contains $V_0 := B(y_0,\tfrac{r}{4})$. We finally verify that there exists a constant $c_0 > 0$ such that
\[
F_\# \Unif(V) \ge c_0 \, \Unif(V_0).
\]
To this end, we recall that $\Unif(V) = |V|^{-1} \,\textbf{1}_{V}(y) \dy$ is just a rescaling of the Lebesgue measure $\textbf{1}_{V}(y)\dy$. Since $F: F^{-1}(V_0) \to V_0$ is bijective, the push-forward of the Lebesgue measure under $F$ satisfies
\begin{align*}
F_\#(\textbf{1}_{V}(y) \dy) \ge F_\#(\textbf{1}_{F^{-1}(V_0)}(y) \dy)
=
|\det DF(F^{-1}(z))|^{-1} \textbf{1}_{V_0}(z) \dz.
\end{align*}
The spectral norm bound $\Vert DF(y) - I_d \Vert_{2} \le \epsilon_0 \le \frac12$, which holds for almost all $y\in V$, now implies that 
\[
\left(\frac{1}{2}\right)^{d} \le  |\det(DF(F^{-1}(z)))| \le \left(\frac{3}{2}\right)^d, 
\quad \dz\text{-almost everywhere}.
\]
Thus,
\begin{align*}
    F_{\#}\Unif(V) 
    &=
    |V|^{-1} F_\#(\textbf{1}_{V}(y) \dy)
    \\
    &\ge 
    |V|^{-1} |\det DF(F^{-1}(z))|^{-1}\textbf{1}_{V_0}(z) \dz
    \\
    &\ge 
    |V|^{-1} \left(\frac{2}{3}\right)^d \textbf{1}_{V_0}(z) \dz
    \\
    &=
    \frac{|V_0|}{|V|} \left(\frac{2}{3}\right)^d \, \Unif(V_0).
\end{align*}
The claim thus follows with $c_0 := \frac{|V_0|}{|V|} \left(\frac{2}{3}\right)^d > 0$.

\end{proof}

\subsection{Proof of Lemma \ref{lem:Lpop1}}
\label{app:Lpop}

\begin{proof}[Proof of  Lemma \ref{lem:Lpop1}]
Before discussing our construction of $A$, we recall that, by definition, $\cA: \U \to L^p(\mu)$ is of the form 
\[
\cA(\Psi) = \cQ(\Psi(u_1), \dots, \Psi(u_N)),
\]
where $u_1,\dots, u_N$ are fixed and $\cQ: \R^N \to L^p(\mu)$ is a reconstruction from point-values. Given $\psi \in U^{\alpha,\infty}_{\bl}$, we now define $A(\psi) \in L^p([0,1]^d)$ by the conditional expectation,
\[
A(\psi)(x)
:= 
\E_{u\sim \mu}\left[\cA(\psi\circ \cE)(u) \cond \cE(u)=x \right], 
\quad \, \forall \, x\in [0,1]^d,
\]
This conditional expectation is well-defined for $\cE_\#\mu$ almost every $x$. By assumption \eqref{eq:c}, we have $\cE_\#\mu \ge c \cdot \Unif([0,1]^d)$, and hence $A(\psi)(x)$ is well-defined for (Lebesgue-) almost every $x\in [0,1]^d$.

We also note that $A(\psi)(x)$ is of the form $A(\psi) = Q(\psi(x_1),\dots, \psi(x_N))$: indeed, by definition, we have $\cA(\psi\circ \cE)(u) = \cQ(\psi(\cE(u_1)),\dots, \psi(\cE(u_N)))(u)$. Hence, upon defining $x_j := \cE(u_j) \in \R^N$, and 
\[
Q(y_1,\dots, y_N)(x) := \E_{u\sim \mu}\left[ 
\cQ(y_1,\dots, y_N)(u) \cond \cE(u) = x
\right],
\]
we then have $A(\psi) = Q(\psi(x_1),\dots, \psi(x_n))$ for all $\psi \in U^{\alpha,\infty}_{\bl}$.

To simplify notation for the following calculations, we define $\Psi_\cA := \cA(\Psi)$ with $\Psi := \psi \circ \cE$. We can then write
\[
A(\psi)(x) = \E_{u}[ \Psi_\cA(u) \cond \cE(u)=x ].
\]
Here $\E_u[\ldots \cond \cE(u)=x ]$ is the conditional expectation over a random variable $u\sim \mu$, with conditioning on $\cE(u) = x$.

For $p\in [1,\infty)$, we then have
\begin{align}
\label{eq:jensens}
|A(\psi)(x)|^p
&= 
\big|
\E_{u}\left[\Psi_\cA(u) \cond \cE(u)=x \right]
\big|^p 
\le
\E_{u}\left[|\Psi_\cA(u)|^p \cond \cE(u)=x \right],
\end{align}
by conditional Jensen's inequality. It follows that
\begin{align*}
\int_{[0,1]^d} |A(\psi)(x)|^p \, dx
&\le
c^{-1}
\int_{\R^d} |A(\psi)(x)|^p \, \cE_\#\mu(dx)
\\
&=
c^{-1}\,
\E_{x\sim\cE_\#\mu}
\Big[
|A(\psi)(x)|^p
\Big]
\\
&\le 
c^{-1}\,
\E_{x\sim\cE_\#\mu}
\left[
\E_u
\Big[
\left|\Psi_\cA(u)\right|^p \, \Big |\, \cE(u)=x
\Big]
\right]
 \\
 &=
c^{-1}\, \E_{u\sim \mu}\Big[
\left|\Psi_\cA(u)\right|^p
\Big].
\end{align*}
The first inequality is by assumption $\cE_\#\mu \ge c \cdot \Unif([0,1]^d)$, the second inequality on the third row is \eqref{eq:jensens} above. The final equality follows from basic properties of the conditional expectation. Thus, recalling that $\Psi_{\cA} = \cA(\Psi) \in L^p(\mu)$, it follows that 
\[
\int_{[0,1]^d} |A(\psi)(x)|^p \, dx
\le
c^{-1} \, \Vert \cA(\Psi) \Vert_{L^p(\mu)}^p < \infty.
\]
This shows that $A: U^{\alpha,\infty}_{\bl} \to L^p([0,1]^d)$ is well-defined. 

It remains to show that $A$ satisfies the claimed lower bound \eqref{eq:lemLpop1}. To see this, we once more apply conditional Jensen's inequality, to obtain
\begin{align*}
\Vert \Psi - \cA(\Psi) \Vert_{L^p(\mu)}^p
&= 
\Vert \Psi - \Psi_\cA \Vert_{L^p(\mu)}^p
\\
&= 
\E_{u\sim \mu} 
\left[
\left|
\psi(\cE(u)) - \Psi_\cA(u) 
\right|^p
\right]
\\
&= 
\E_{x\sim \cE_\#\mu} 
\E_{u}
\left[
\Big|
\psi(\cE(u)) - \Psi_\cA(u) 
\Big|^p 
\cond \cE(u)=x\right]
\\
&\ge 
\E_{x\sim \cE_\#\mu} 
\left[
\Big|
\psi(x) - \E_{u}\left[\Psi_\cA(u)  \cond \cE(u)=x\right]
\Big|^p
\right]
\\
&= 
\E_{x\sim \cE_\#\mu} 
\left[
\left|
\psi(x) - A(\psi)(x)
\right|^p
\right]
\\
&=
\Vert
\psi - A(\psi)
\Vert_{L^p(\cE_\#\mu)}^p.
\end{align*}
Recalling \eqref{eq:c}, this implies that 
\[
\Vert \Psi - \cA(\Psi) \Vert_{L^p(\mu)}^p
\ge c \Vert
\psi - A(\psi)
\Vert_{L^p([0,1]^d)}^p.
\]
Since $\psi \in U^{\alpha,\infty}_{\bl}$ was arbitrary and $\Psi = \psi \circ \cE$, the proof of \eqref{eq:lemLpop1} is complete.
\end{proof}

\subsection{Proof of Lemma \ref{lem:onto-param}}
\label{app:onto-param}

\begin{proof}[Proof of Lemma \ref{lem:onto-param}]
By assumption on $F_\xi = F(\slot;\xi)$, we have
\begin{align*}
 \Vert F_\xi - \id \Vert_{W^{1,\infty}(V)} \le \epsilon_0, \quad \forall \, \xi \in K.
\end{align*}
By Lemma \ref{lem:onto}, there exists $V_0\subset V$ and a constant $c_0>0$, such that 
\begin{align}
\label{eq:fxipush}
(F_\xi)_\# \Unif(V) \ge c_0 \, \Unif(V_0), \quad \forall \, \xi \in K.
\end{align}
We want to show that the push-forward under $F$ of the product measure $\unif(V)\otimes \P \in \cP(V\times \Omega)$ satisfies,
\[
F_\#(\unif(V)\otimes \P) \ge c_0 \P(K) \,  \Unif(V_0).
\]
Given a non-negative, bounded measurable function $\phi: \R^d \to [0,\infty)$, we have
\begin{align*}
\E_{z\sim F_\#(\Unif(V)\otimes \P)} [\phi(z)]
&= \E_{(y,\xi) \sim \Unif(V)\otimes \P}\left[ \phi(F(y;\xi)) \right]
\\
&=  \E_{\xi \sim \P}\left[ 
\E_{y\sim \Unif(V)}\left[ \phi(F_\xi(y)) \right]
\right]
\end{align*}
By \eqref{eq:fxipush}, we have
\[
\E_{y\sim \Unif(V)}\left[ \phi(F_\xi(y)) \right]
\ge 
c_0 \, \E_{z\sim \Unif(V_0)}\left[ \phi(z) \right], \quad \forall \, \xi \in K,
\]
and hence 
\begin{align*}
\E_{z\sim F_\#(\Unif(V)\otimes \P)} [\phi(z)]
&\ge 
\E_{\xi \sim \P}\left[  \textbf{1}_{K}(\xi) \,
\E_{y\sim \Unif(V)}\left[ \phi(F_\xi(y)) \right]
\right]
\\
&\ge
c_0 \, \E_{\xi \sim \P}\left[ \textbf{1}_{K}(\xi) \, \E_{z\sim \Unif(V_0)}\left[ \phi(z) \right]
\right]
\\
&=
c_0  \P(K) \, \E_{z\sim \Unif(V_0)}\left[ \phi(z) \right].
\end{align*}
Since $\phi \ge 0$ was arbitrary, the claim follows.
\end{proof}

\subsection{Proof of Proposition \ref{prop:encoder-don}}
\label{app:encoder-don}

\begin{proof}[Proof of Proposition \ref{prop:encoder-don}]
Let $I_1,\dots,I_d$ denote the open intervals in Assumption \ref{ass:mud}. Define $V := I_1\times \dots \times I_d\subset \R^d$. Instead of proving \eqref{eq:encoder-don} directly, we will consider the simplified encoder $\cE: \cX \to \R^d$, with components of the form
\begin{align}
\label{eq:simpleenc}
\cE_j(u) = \sum_{k=1}^{d_0} c_{jk} \ell_k(u),
\end{align}
and where the coefficient $c_{jk}$ are chosen to ensure \eqref{eq:enc-approx} for a $\delta > 0$ to be determined. Our goal is to show that there exists an open set $V_0 \subset V$ and constant $c>0$, such that 
\begin{align}
\label{eq:simpleV0}
\cE_\#\mu \ge c \, \unif(V_0).
\end{align}
The general claim \eqref{eq:encoder-don} then follows by introducing a scaling factor $\gamma > 0$ and bias $b\in \R^d$, such that 
\[
[0,1]^d \subset \gamma \cdot V_0 + b,
\]
and replacing the simple encoder $\cE$ \eqref{eq:simpleenc} by 
\[
\tilde{\cE}_j(u) := b_j + \sum_{k=1}^{d_0} a_{jk} \ell_k(u),
\]
where $a_{jk} := \gamma \, c_{jk}$. Thus, it only remains to show \eqref{eq:simpleV0}.

Fix $\delta > 0$ for the moment. We will determine conditions on $\delta$ which imply that \eqref{eq:simpleV0} holds for the encoder \eqref{eq:simpleenc}, where $c_{jk}$ are chosen according to \eqref{eq:enc-approx}. With this specific choice of $c_{jk}$, we can then write \eqref{eq:enc-approx} in the form 
\begin{align}
\label{eq:Ej-don}
\left\Vert 
e^\ast_j - \cE_j 
\right\Vert_{\cX^\ast} 
\le \delta, \quad \forall \, j=1,\dots, d.
\end{align}
Our goal is to apply Lemma \ref{lem:onto-param}. To this end, we recall the decomposition $u = u(y;\xi) = \xi + \sum_{j=1}^d y_j e_j$ of \eqref{eq:muparam} and the probability measures $\mu_d \in \cP(\R^d)$, $\mu_d^\perp \in \cP(\Omega_d)$ of Assumption \ref{ass:mud}. Given this decomposition, we let $F: V\times \Omega_d \to \R^d$ be defined by $F(y;\xi) := \cE(u(y;\xi))$, and denote $F_\xi(y) := F(y;\xi)$. We will apply Lemma \ref{lem:onto-param} with this choice of $F_\xi$ and with $\P:= \mu_d^\perp$. 
As our final ingredient, we recall for any probability measure on a Banach space $\cX$, that there exists a set $B>0$, such that $\P(K) > 0$ for $K:=\set{ \xi \in \Omega_d}{\Vert \xi \Vert_{\cX}\le B}$. Given these preparatory remarks, our goal now is to show that \eqref{eq:onto-param} holds, provided that $\delta>0$ in \eqref{eq:enc-approx} is sufficiently small.

To this end, we first note that $(y,\xi) \mapsto F(y;\xi)$ is linear, and hence $y \mapsto F_\xi(y)$ is affine for fixed $\xi$, and 
\[
F_\xi(y) = F(y;\xi) = F(y;0) + F(0,\xi) = F_0(y) + F_\xi(0).
\]
Thus, we can write
\begin{align}
\label{eq:fxi-don}
\Vert F_\xi(y) - y \Vert_{W^{1,\infty}(V)}
&\le
\Vert F_0(y) - y \Vert_{W^{1,\infty}(V)}
+
\Vert F_\xi(0) \Vert_{W^{1,\infty}(V)}.
\end{align}
We will bound both terms on the right, individually.

The last term is constant in $y$, and hence $\Vert F_\xi(0) \Vert_{W^{1,\infty}(V)} = |F_\xi(0)|$. By \eqref{eq:simpleenc}, the $j$-th component of $F_\xi(0) = \cE(u(0;\xi)) = \cE(\xi)$ is given by,
\[
\cE_{j}(\xi) 
:=
\sum_{k=1}^{d_0} c_{jk} \ell_k(\xi).
\]
Since $j\le d$, it follows that $e^\ast_j(\xi) = 0$. From \eqref{eq:Ej-don}, we conclude that 
\begin{align*}
|\cE_{j}(\xi)|
&=
\left|
\cE_j(\xi) - e_j^\ast(\xi)
\right|
\\
&\le
\left\Vert 
\cE_j -  e_j^\ast 
\right\Vert_{\cX^\ast} 
\Vert 
\xi
\Vert_{\cX}
\\
&\le 
\delta \Vert 
\xi
\Vert_{\cX}.
\end{align*}
For  $\xi \in K = \set{\xi}{ \Vert \xi \Vert_\cX \le B }$, then it follows that
\[
|F_\xi(0)| 
\le d \max_{j\in [d]} |\cE_j(\xi)|
\le d B \, \delta.
\]
Thus, for $\delta \le \epsilon_0 / (3dB)$, we obtain 
\begin{align}
\label{eq:f0-1}
\Vert F_\xi(0)\Vert_{W^{1,\infty}(V)} \equiv |F_\xi(0)| \le \epsilon_0/3,
\quad \forall \, \xi \in K.
\end{align}
This provides our estimate for the second term in \eqref{eq:fxi-don}. To bound the first term, we note that $F_0(y) = \cE(u(y;0))$, and hence
\begin{align*}
|F_0(y)_j - y_j|
&= 
|\cE_j(u(y;0)) - e_j^\ast(u(y;0))|
\\
&\le 
\Vert \cE_j - e_j^\ast \Vert_{\cX^\ast} \Vert u(y;0) \Vert_{\cX}.
\end{align*}
Since $y\in V$ is from a bounded set, we can assume without loss of generality that $B>0$ is chosen sufficiently large such that $\Vert u(y;0) \Vert_{\cX} \le B$ for all $y\in V$, and hence, we obtain,
\begin{align}
\label{eq:f0-2}
\Vert F_0(y) - y \Vert_{L^\infty(V)} \le \epsilon_0/3,
\end{align}
whenever $\delta \le \epsilon_0/(3dB)$. It remains to derive a similar bound on 
\[
\Vert D_y F_0(y) - D_y y \Vert_{L^\infty(V)}
=
\Vert D_y F_0(y) - I \Vert_{L^\infty(V)}.
\]
To this end, we recall that $y \mapsto F_0(y)$ is linear, and hence is represented by a matrix $A\in \R^{d\times d}$, i.e. $F_0(y) = Ay$. It follows that 
\[
\Vert D_y F_0(y) - D_y y \Vert_{L^\infty(V)} 
=
\Vert A - I \Vert_{2},
\]
where $\Vert \slot \Vert_2$ is the operator norm. Retracing the argument above, it follows that any $y\in \R^d$, we have
\begin{align*}
\left| Ay - y\right|_{\ell^2} 
&\le
d \max_{j\in [d]} |\cE_j(u(y;0)) - e_j^\ast(u(y;0)) | 
\\
&\le 
d \Vert \cE_j - e_j^\ast \Vert_{\cX^\ast} \Vert u(y;0) \Vert_{\cX}
\\
&\le d \delta  \Vert u(y;0) \Vert_{\cX}.
\end{align*}
We can furthermore find a constant $C>0$, depending only on $d$ and $e_1,\dots, e_d$, such that 
\[
 \Vert u(y;0) \Vert_{\cX} \le C |y|_{\ell^2},
 \quad
 \forall \, y \in \R^d.
\]
Thus, 
\[
\left| Ay - y\right|_{\ell^2} 
\le dC \delta \,  |y|_{\ell^2},
\]
which, upon taking the supremum over all $|y|_{\ell^2}=1$, implies that 
\[
\Vert A - I \Vert_{op} \le (dC)\, \delta.
\]
Hence, for $\delta \le \min(\epsilon_0 / (3dC), \epsilon_0 / (3dB))$, we conclude that 
\[
\Vert D_y F_0(y) - D_y y \Vert_{L^\infty(V)}
\le \epsilon_0/3,
\]
and by \eqref{eq:f0-2},
\begin{align*}
\Vert F_0(y) - y \Vert_{W^{1,\infty}(V)}
&=
\Vert D_y F_0(y) - D_y y \Vert_{L^\infty(V)}
+
\Vert F_0(y) - y \Vert_{L^\infty(V)}
\\
&\le
 2\epsilon_0/3.
\end{align*}
Combining the last estimate, \eqref{eq:f0-1} and \eqref{eq:fxi-don}, we conclude that
\[
\Vert F_\xi(y) - y \Vert_{W^{1,\infty}(V)}
\le \epsilon_0,
\quad 
\forall \, \xi \in K.
\]
Thus, by Lemma \ref{lem:onto-param}, there exists $V_0\subset V$ and $c_0>0$, such that
\begin{align*}
\cE_\#\mu 
&=
F_\#(\mu_d \otimes \P)
\\
&\ge
c_\rho
F_\#(\Unif(V) \otimes \P)
\\
&\ge 
\underbrace{c_\rho c_0 \P(K)}_{=: c} \, \unif(V_0).
\end{align*}
\end{proof}

\subsection{Proof of Lemma \ref{lem:dilation}}
\label{app:dilation}

\begin{proof}[Proof of Lemma \ref{lem:dilation}]
    By the assumption that $f\in U^{\alpha,\infty}_{\bl}(D)$ it follows that for each $n\in \N$ there is $\psi_n\in \Sigma_n^{\bl}$ with 
    \begin{equation}\label{eq:initialapprox}
        \|f - R_\sigma(\psi_n)\|_{L^\infty(D)}\le n^{-\alpha}.   
    \end{equation}
    Recall that, by definition, $\psi_n= \left((A_1,b_1),\dots , (A_L,b_l)\right)\in \bigtimes_{l=1}^L (\R^{d_l \times d_{l-1}}\times \R^{d_l})$ with some $(d_0,d_1,\dots , d_L)\in \N^{L+1}$, where $d_0 = d$, $L\le \bl(n)$, $W(\psi_n)\le n$ and $\|\psi_n\|_{\NN}\le 1$.
    
    Consider the functions 
    \begin{equation}\label{eq:hndef}
        h_n:\left\{ \;
        \begin{aligned}
        \R^e & \to \R, \\ 
        x & \mapsto R_\sigma(\psi_n)(Cx + b),
        \end{aligned}
        \right.
    \end{equation}
    which clearly satisfy that 
    \begin{equation}
        h_n = R_\sigma (\varphi_n),
    \end{equation}
    where 
    \begin{equation*}
        \varphi_n = \left(\left(\tilde{A}_l,\tilde{b}_l\right)\right)_{l=1}^L\in  \left(\R^{e\times d_1}\times \R^{d_1}\right)\times \left(\bigtimes_{l=2}^L \left(\R^{d_l\times d_{l-1}}\times \R^{d_l}\right)\right),
    \end{equation*}
    and 
    \begin{equation*}
        \tilde{A}_1 = A_1C,\ \tilde{b}_1 = A_1b + b_1\quad \mbox{and}\quad
        \tilde{A}_l = A_l,\ \tilde{b}_l=b_l\quad \mbox{for }l=2,\dots , L. 
    \end{equation*}
    This, and the fact that $W(\psi_n)\le n$ and $\|\psi_n\|_{\NN}\le 1$ readily yields that 
    \begin{equation}\label{eq:dilationsize}
        W(\varphi_n)\le W(\psi_n)\cdot \left(\|C\|_{\ell^0} + \|b\|_{\ell^0}\right) = n\cdot \left(\|C\|_{\ell^0} + \|b\|_{\ell^0}\right). 
    \end{equation}
    and 
    \begin{equation}\label{eq:dilationnorm}
        \|\varphi_n\|_{\NN}\le \max\left(\{1\}\cup \bigcup_{k=1}^e\left\{\sum_{j=1}^d |C_{j,k}|\right\}\cup \left\{1 + \sum_{j=1}^d|b_j|\right\}\right)=:T.
    \end{equation}
    Note that $T$ is independent of $n$.
    Now pick $R\geq T$ to be determined later and denote
    \begin{equation*}
        \tau_n:= \left(\left(\frac1R \tilde{A}_1,\frac1R\tilde{b}_1\right),\left({A_2},\frac1R{b}_2\right),\dots ,\left({A_L},\frac1R{b}_L\right)\right).   
    \end{equation*}
    By the homogeinity of the ReLU activation function $\sigma$ it holds that 
    \begin{equation}\label{eq:scaled_hn_as_nn}
        R_\sigma(\tau_n) = \frac{1}{R}h_n. 
    \end{equation}
    Moreover, due to \eqref{eq:dilationnorm}, the fact that $R\geq \max\{1,T\}$ and \eqref{eq:dilationsize} it holds that
    \begin{equation}
        \|\tau_n\|_\NN \le 1\quad \mbox{and} \quad W(\tau_n)\le n\cdot \left(\|C\|_{\ell^0} + \|b\|_{\ell^0}\right),
    \end{equation}
    which implies that
    \begin{equation}\label{eq:tau_in_approxset}
        \tau_n\in \Sigma_{n\cdot \left(\|C\|_{\ell^0} + \|b\|_{\ell^0}\right)}^{\bl \left(\cdot/\left(\|C\|_{\ell^0} + \|b\|_{\ell^0}\right)\right)}\subset \Sigma_{n\cdot \left(\|C\|_{\ell^0} + \|b\|_{\ell^0}\right)}^{\bl},
    \end{equation}
    where the last inclusion follows from the fact that $\bl$ is non-decreasing.
    
    Moreover, by \eqref{eq:initialapprox}, \eqref{eq:hndef}, \eqref{eq:scaled_hn_as_nn} and the definition of $E$ it holds that
    \begin{equation}\label{eq:tauapprox}
        \left\|\frac1Rf(C\cdot + b) - R_\sigma(\tau_n)(\cdot)\right\|_{L^\infty(E)}
        \le \frac{1 }{R} \cdot n^{-\alpha}
    \end{equation}
    for all $n\in \N$.

    Let $m\in  \N$ be arbitrary and define $\mu_m:=\tau_{\lfloor m/(\|C\|_{\ell^0} + \|b\|_{\ell^0}) \rfloor}$. 
    Equations \eqref{eq:tau_in_approxset} and \eqref{eq:tauapprox}
    now readily yield that $\mu_m\in \Sigma_m^{\bl}$ and 
    $$
        \left\|\frac1Rf(C\cdot + b) - R_\sigma(\mu_m)(\cdot)\right\|_{L^\infty(E)}
        \le \frac{2^\alpha\cdot \left(\|C\|_{\ell^0} + \|b\|_{\ell^0}\right)^\alpha }{R} m^{-\alpha}.
    $$
    By choosing $R$ sufficiently large this implies that
    $$
        d_{L^\infty(E)}\left(\frac1R f(C\cdot + b) , \Sigma_m^{\bl}\right)\le m^{-\alpha}\quad 
        \mbox{for all} \quad m\in \N,
    $$
    which implies that $f(C\cdot + b) \in R\cdot U_{\bl}^{\alpha, \infty}(E)$, as claimed.
\end{proof}

\subsection{Proof of Lemma \ref{lem:inclusion-no}}
\label{app:inclusion-no}

\begin{proof}
By construction, the encoder $\cE$ is of the form 
\[
\cE(u) = \aint_{D} R(u(x),x) \, dx,
\]
where $R: \R \times \R^{d_D} \to \R^d$ is a shallow neural network. Let $\psi = ((A_1,b_1), \dots, (A_L,b_L))$ be a neural network with $L \le \bl_0$ layers. We define a shallow neural network,
\[
\tilde{R}(\eta,x) := A_1 R(\eta,x) + b_1,
\]
and $\tR(u)(x) := \tilde{R}(u(x),x)$. We next define the first hidden ANO layer as,
\[
\cL_1(v)(x) 
=
\sigma\left( 
\fint_D v(x) \, dx
\right),
\]
i.e. a hidden layer with weight matrix and bias $W_1 =0$, $b_1 = 0$. For $j=2,\dots, L-2$, we define
\[
\cL_j(v)
=
\sigma\left(
(A_j-I)v(x) + b_j + \fint_D v(x) \, dx
\right),
\]
where $I$ denotes the unit matrix
and finally,
\[
Q(v) = A_L \sigma\left( A_{L-1} v + b_{L-1} \right) + b_L.
\]
Let $\psi_1(\xi) = \sigma(A_1 \xi + b_1)$ denote the first layer of $\psi$. Then we have
\[
\psi_1(\cE(u)) 
= 
\sigma \left( A_1 \fint_D R(u(x),x) \, dx + b_1 \right)
=
\sigma\left( \fint_D \tilde{R}(u(x),x) \, dx \right)
=
\cL_1(\tR(u)).
\]
Since the output $v(x) := \psi_1(\cE(u)) = \cL_1(\tR(u))$ is a constant function, it follows that
\[
\cL_2(v) = \sigma \left( (A_2 - I)v(x) + b_2 + \fint_D v(x) \, dx \right)
= \sigma \left( A_2 v + b_2 \right).
\]
Thus, $\cL_2(\cL_1(\tR(u))) = \sigma\left( A_2 \psi_1(\cE(u)) + b_2 \right)$ agrees with the output of the second hidden layer of $\psi$. Continuing recursively, it follows that
\[
\Psi(u) := \cQ \circ \cL_{L-2} \circ \dots \circ \cL_1 \circ \tR(u) 
=
\psi \circ \cE(u),\quad \forall \, u \in \cX(D).
\]
Thus, $\psi\circ \cE$ is equal to an ANO of depth $L-2 \le \bl_0-2$, with input layer $\tR$ and output layer $\cQ$. Employing a rescaling argument similar to the proof of Lemma \ref{lem:dilation}, relying on the homogeneity of ReLU as well as the fact that the total number of layers is bounded by $\bl_0< \infty$, it follows that for sufficiently large $\gamma>0$, depending only on $\cE$, $\bl_0$ and $\alpha$, we have $\frac{1}{\gamma} \Psi \in \U^{\alpha}_{\bl,\NO}$, i.e. $\Psi \in \gamma\cdot \U^{\alpha}_{\bl,\NO}$. Here we have also made use of the fact that $\bl_0-2 \le \bl^\ast$.
\end{proof}



\bibliographystyle{plain}
\bibliography{references}


\end{document}